\documentclass{article}

\usepackage[numbers]{natbib}
\usepackage[preprint]{neurips_2023}

\usepackage{amsmath,amsfonts,bm}
\usepackage{amsthm}
\usepackage{bbm}

\def\eqref#1{equation~(\ref{#1})}

\def\1{\bm{1}}

\def\rvz{{\mathbf{z}}}

\def\rmG{{\mathbf{G}}}
\def\rmH{{\mathbf{H}}}

\def\rmT{{\mathbf{T}}}

\def\vtheta{{\bm{\theta}}}
\def\vgamma{{\bm{\gamma}}}
\def\vGamma{{\bm{\Gamma}}}

\def\mI{{\bm{I}}}

\DeclareMathAlphabet{\mathsfit}{\encodingdefault}{\sfdefault}{m}{sl}
\SetMathAlphabet{\mathsfit}{bold}{\encodingdefault}{\sfdefault}{bx}{n}

\def\sR{{\mathbb{R}}}

\newcommand{\sigmoid}{\sigma}

\newcommand{\vtc}{\vtheta(t)}
\newcommand{\vtd}{\vtheta_t}
\newcommand{\ldsrc}{L^{\text{dsr}}(\vtc)}
\newcommand{\ldsrd}{L^{\text{dsr}}(\vtd)}
\newcommand{\simpldsrc}{L^{\text{dsr}}(t)}
\newcommand{\rmd}{\mathrm{d}}
\newcommand{\trn}{\text{trn}}
\newcommand{\dsr}{\text{dsr}}
\newcommand{\gnc}{\gamma_n(t)}
\newcommand{\gnd}{\gamma_{n,t}}
\newcommand{\Gnc}{\Gamma_n(t)}
\newcommand{\Gn}{\Gamma_n}
\newcommand{\CR}{\mathrm{CR}}
\newcommand{\SNR}{\mathrm{SIM}}
\newcommand{\con}{\mathrm{CT}}
\newcommand{\Const}{\mathrm{Const}}
\newcommand{\vga}{\vgamma}

\usepackage[dvipsnames]{xcolor}         
\definecolor{linkColor}{rgb}{0.2,0.4,0.6}
\usepackage[utf8]{inputenc} 
\usepackage[T1]{fontenc}    
\usepackage[colorlinks=true,linkcolor=linkColor,citecolor=linkColor,filecolor=linkColor,urlcolor=linkColor]{hyperref}       
\usepackage{url}            
\usepackage{booktabs}       
\usepackage{amsfonts}       
\usepackage{nicefrac}       
\usepackage{microtype}      
\usepackage{xcolor}         

\usepackage{hyperref}
\usepackage{url}
\usepackage{multirow}
\usepackage{booktabs}
\usepackage{graphicx}
\usepackage{wrapfig}
\usepackage{bm}
\usepackage{enumitem}
\usepackage{tcolorbox}
\usepackage{makecell}
\usepackage{subfigure}
\usepackage{diagbox}
\usepackage{makecell}
\usepackage{algorithm}
\usepackage{algorithmic}
\usepackage{bbding}

\newtheorem{property}{Property}[section]

\usepackage{amsmath}
\usepackage{amssymb}
\usepackage{mathtools}
\usepackage{amsthm}

\usepackage[capitalize,noabbrev]{cleveref}

\theoremstyle{plain}
\newtheorem{theorem}{Theorem}[section]

\newtheorem{lemma}[theorem]{Lemma}

\theoremstyle{definition}

\theoremstyle{remark}
\newtheorem{remark}{Remark}
\usepackage[textsize=tiny]{todonotes}

\title{Towards Optimal Learning of Language Models}

\author{%
Yuxian Gu$^{1,2}$\thanks{Contribution during an internship at Microsoft Research.~\small $\langle$\texttt{guyx21@mails.tsinghua.edu.cn}$\rangle$},~~~\ Li Dong$^2$,~~~\ Yaru Hao$^2$,~~~\ Qingxiu Dong$^2$,~~~\ Minlie Huang$^1$,~~~\ Furu Wei$^2$ \\
$^1$The CoAI Group, Tsinghua University \\
$^2$Microsoft Research \\
{\href{https://aka.ms/GeneralAI}{https://aka.ms/GeneralAI}} \\
}

\begin{document}

\maketitle

\vspace{-0.5cm}
\begin{abstract}
This work studies the general principles of improving the \textbf{learning} of language models (LMs), which aims at reducing the necessary training steps for achieving superior performance. Specifically, we present a theory for the \textit{optimal learning} of LMs. We first propose an objective that optimizes LM learning by maximizing the \textit{data compression ratio} in an ``LM-training-as-lossless-compression'' view. Then, we derive a theorem, named \textit{Learning Law}, to reveal the properties of the dynamics in the optimal learning process under our objective. The theorem is then validated by experiments on a linear classification and a real-world language modeling task. Finally, we empirically verify that the optimal learning of LMs essentially stems from the improvement of the coefficients in the scaling law of LMs, indicating great promise and significance for designing practical learning acceleration methods. Our code can be found at \url{https://aka.ms/LearningLaw}.
\end{abstract}


\begin{figure*}[ht]
\centering
\includegraphics[width=\textwidth]{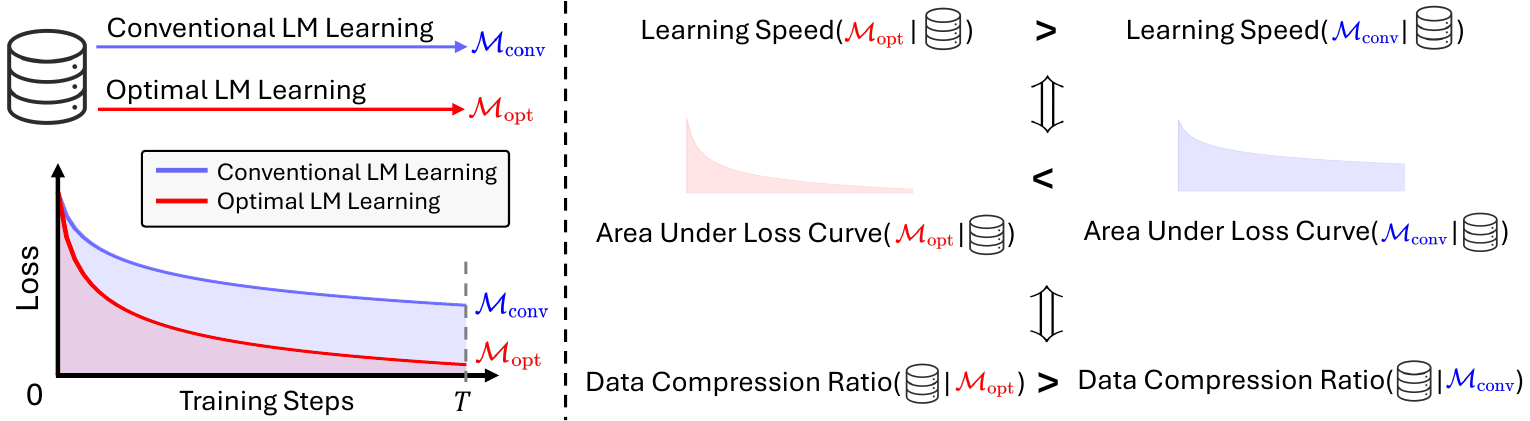}
\caption{Our objective is to minimize the area under loss curve, which is equivalent to maximizing the compression ratio of training corpus in the ``LM-training-as-lossless-compression'' view. A learning law is proposed to reveal the training dynamics of the above optimal learning.}
\label{fig:obj}
\end{figure*}

\begin{figure}[ht]
\begin{minipage}[ht]{0.44\textwidth}
\centering
\includegraphics[width=0.98\textwidth]{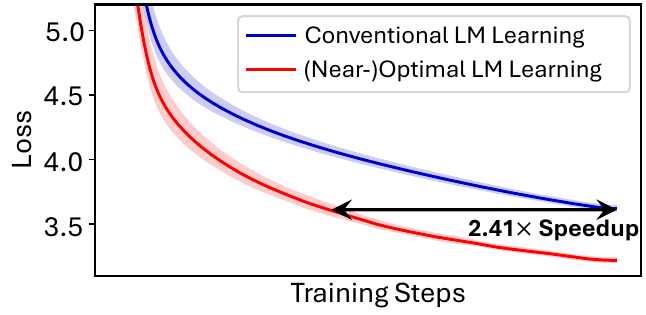}
\caption{Optimal learning gets the theoretical speedup upper bound of Transformer LM training on TinyStories corpus~\cite{tinystories}.}
\label{fig:exp}
\end{minipage}\hspace{0.5cm}
\begin{minipage}[ht]{0.51\textwidth}
\centering
\small
\vspace{0.3cm}
\begin{tabular}{l|cccccccc}
\toprule                 
\bf Scaling Laws  & $B$  & $\beta$  \\ \midrule 
Conventional LM Learning & $\text{3.16}\times \text{10}^\text{8}$ & 0.12  \\
(Near-)Optimal LM Learning & $\textbf{1.99} \bm{\times}\textbf{10}^{\textbf{7}}$ & \textbf{0.14}  \\
\bottomrule
\end{tabular}
\vspace{0.3cm}
\makeatletter\def\@captype{table}\makeatother\caption{The (near-)optimal LM learning improves the scaling laws~\cite{scaling_law} over conventional LM training. The coefficients $B,\beta$ are used to fit the loss curves in Figure \ref{fig:exp}, i.e., $\mathrm{Loss} = L_0 + \left( B/t \right)^{\beta}$ when $t>t_0$. See Section \ref{sec:scaling_law} for details.}
\label{tab:scale_coef}
\end{minipage}

\end{figure}

\section{Introduction}

With the thriving of language models (LMs;~\citealp{plmsurvey,foundation_model}), there is an increasing focus on improving the \textbf{learning}~\cite{optimal_learning,eighty_five} of LMs, which aims at accelerating the learning speed and achieving a certain model performance with as few training steps as possible~\cite{llm_eficiency_survey}. This focus helps humans explore the limits of LMs given the rapid growth of their computational demand~\cite{chinchila}, and promotes democratization~\cite{h2oGPT} of large language models (LLMs; \citealp{gpt3,chatgpt,gpt4,palm,palm2}), which is valuable for both research communities and industry sectors~\cite{llama,llama2,mistral}.

In this paper, we present a theory for optimal learning of LMs. Unlike prior works exploring practical acceleration methods at the model-level~\cite{ln_study,layer_drop}, optimizer-level~\cite{lamb,sophia}, or data-level~\cite{d4,semdedup,doremi}, our work demonstrates the principles of optimizing the LM learning speed, including the optimization objective, the property of optimal learning dynamics, and the essential improvement of the learning acceleration. 


Specifically, for the optimization objective, we propose to minimize the area under the loss curve (AUC;~\citealp{auc}), which has a clear physical significance: the \textit{description length} when we view the next-token-prediction LM training process as lossless compression of the training data~\cite{nncp,trm_text_compress,jack_rae_compression}. 
As shown in Figure \ref{fig:obj}, a learning process with the smallest loss AUC corresponds to the highest compression ratio. Simultaneously, the loss in this process also converges to a small value at the highest rate, given sufficiently large total training steps.
Therefore, we consider \textbf{optimizing LM learning equivalent to maximizing the corresponding compression ratio of the learning process}, and adopt the latter as the optimization objective in our theory. Similar objectives are also employed to interpret the remarkable generalization performance of recent LLMs~\cite{llmzip,lm_is_compression}.

We then derive a theorem, named \textit{Learning Law}, that characterizes the property of dynamics in the LM learning process that achieves the optimum of our objective. 
Here, a learning process is induced by a \textit{learning policy} that determines which data points the LM learns as the training progresses. In this way, we solve the optimal learning policy in the sense that the corresponding compression ratio is maximized, and obtain our Learning Law (see Theorem \ref{trm:main} for a formal expression):
\begin{tcolorbox}[title=Learning Law]
\bf
\textit{All examples have the same contribution to the LM in the optimal learning process.}
\end{tcolorbox}
\begin{figure*}[ht]
    \centering
        \includegraphics[width=0.96\textwidth]{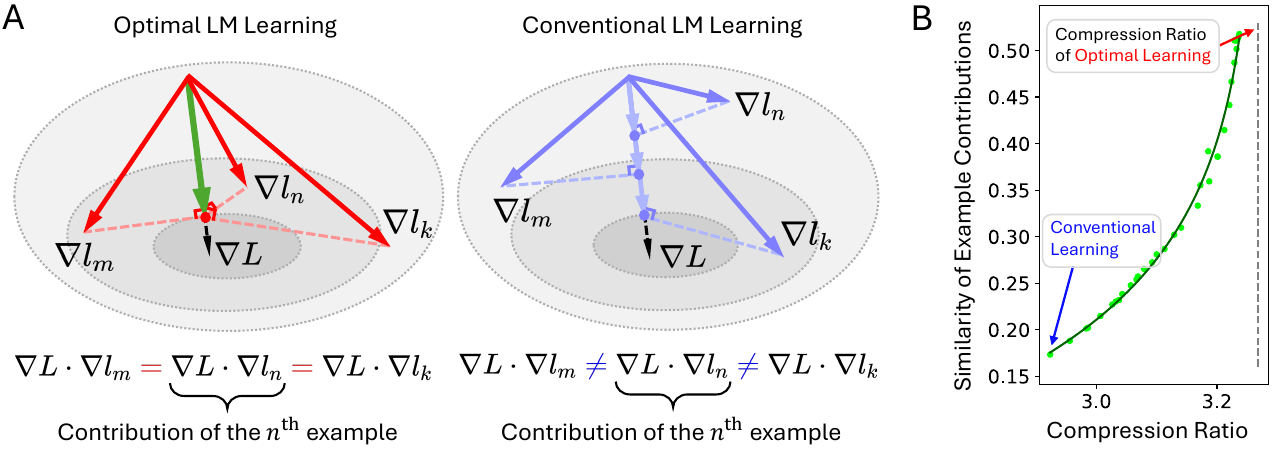}
    \caption{\textbf{A}: 3-D illustration of Learning Law (Theorem \ref{trm:main}). In the optimal learning process, all training examples should have the same contribution to LM learning, where the contribution is defined as the dot-product of the gradient on individual samples ($\nabla l_m$, $\nabla l_n$, and $\nabla l_k$) and the gradient of a desired loss ($\nabla L$). See Section \ref{sec:derive} for rigorous notation definitions. \textbf{B}: Experimental evidence of Learning Law. When LM learning approaches the optimum, the similarity of example contributions tends to $+\infty$, which means all examples have the same contribution to the LM.}
    \label{fig:learning law}
\end{figure*}
As shown in Figure \ref{fig:learning law}, the contribution of an example is defined as the dot-product of its gradient and the gradient of a desired loss\footnote{Note that the desired loss is not necessarily the same as the training loss as discussed in Section \ref{sec:prelim}.}
, which measures its influence on the LM in the desired learning direction. 
Learning Law also suggests a \textit{matching of local and global learning speed} in the optimal learning process, which interprets the optimal learning policy as a dynamic data re-weighting strategy that encourages the LM to learn highly contributive examples and simultaneously avoid over-fitting them. Similar mechanisms are also found critical to the best teaching methods for humans in psychological research~\cite{cog_teaching,goldilocks}.

We examine our theory by experiments on linear classification tasks based on Perceptron\footnote{In Appendix \ref{app:perceptron_as_comp}, we provide a lossless data compression view of the Perceptron training, indicating that our theory also applies.}~\cite{perceptron}
and real-world language modeling tasks based on Transformer~\cite{transformer}.
We first design a gradient-based method to search for the optimal learning policy under our objective. 
Then, we verify that the dynamics of the learning process induced by the found near-optimal policy aligns well with our Learning Law.
Finally, as shown in Table \ref{tab:scale_coef}, we provide empirical evidence showing that the near-optimal learning policy essentially improves the coefficients in the training step scaling law of LMs~\cite{scaling_law}, which leads to 5.50$\times$ and 2.41$\times$ speedup to Perceptron and Transformer learning, respectively.
This emphasizes the promise and significance of exploring more scalable methods to optimize the learning policy in practice and accelerate the training of LLMs.

\section{Problem Formulation}
\label{sec:prelim}

We consider LM training on a large-scale dataset with $N$ examples $\{x_n^\trn\}_{n=1}^N$ for a sufficiently large total training time steps $T$. Let $\vga_{n, t}$ denote the weight of the $n^{\text{th}}$ training example at the time step $t$,  a \textbf{learning policy} is represented by a time-variant distribution over $N$ training examples $\vga_t=\left[\gamma_{1,t}, \gamma_{2,t}, \cdots, \gnd\right]^\top$, satisfying $\sum_{n=1}^N\gnd=1$ and $\gnd\ge0$ for $1\le n \le N, 0\le t \le T-1$. The conventionally trained LM learns with a policy $\gamma^c_{n,t}=\frac{1}{N}$ (\textbf{conventional learning}). 
Recent works~\cite{tracin, tracein_lms} have shown that theories derived based on Gradient Decent (GD) offer insights into other gradient-based algorithms~\cite{adam}. Therefore, for simplicity, we assume the LM is trained with GD for $t=0,1,\cdots, T-1$:
\begin{equation}
\begin{aligned}
    L^\trn_t(\vtd) &= \sum^{N}_{n=1}\gnd l(x_n^\trn, \vtd), \\
    \vtheta_{t+1} &= \vtd - \eta \nabla L^\trn_t(\vtd),
    \label{eq:gd}
\end{aligned}
\end{equation}
where $\vtd\in \sR^D$ is the model parameters flattened into a $D$-dimensional vector at the time step $t$, $\eta$ is the learning rate, and $l(\cdot, \cdot)$ is the loss function of the learning problem. For LMs, $l(\cdot, \cdot)$ is typically the Maximum Likelihood Estimation (MLE) loss: $l(x, \vtd) = -\log p_{\vtd}(x)$, where $x$ is a text sequence. 
Following \cite{data-selection-IS} and \cite{rho-loss}, we focus on the learning speed reflected by the reduction rate of a desired loss $L^\dsr$ computed on $K$ examples $\{x_k^\dsr\}_{k=1}^K$ that \textit{do not necessarily} follow the same distribution as the training examples:
\begin{equation}
\vspace{-0.1cm}
\ldsrd = \frac{1}{K}\sum^{K}_{k=1} l(x^\dsr_k, \vtd).
\label{eq:tg}
\end{equation}
This formulation applies to a broad of practical scenarios including classical machine learning using a validation set to prevent over-fitting~\cite{nature_of_stat_learning}, large-scale pre-training relying on a carefully curated held-out corpus to evaluate generalization performance~\cite{scaling_law}, and domain adaptation where a natural difference exists between training and target distribution~\cite{data-selection-IS}. As such, we search for the learning policy $\vga_t$ that maximizes the reduction rate of $\ldsrd$ to optimize LM learning.

However, direct analysis of this optimization problem is difficult due to the discreteness of GD.
Therefore, we focus on the \textit{continuous limit} of GD by considering the corresponding gradient flow of Equation \ref{eq:gd} for $t\in [0,T]$, which is more amenable to theoretical analysis~\cite{reg_sgd}:
\begin{equation}
    \label{eq:continuous}
    \frac{\rmd }{\rmd t}\vtc = -\nabla L^\trn(\vtc, t) = -\nabla \sum^{N}_{n=1}\gamma_{n}(t)l(x_n^\trn, \vtc),
\end{equation}
where $\gnc$ is a smooth interpolation function of $\gnd$. According to the results in numerical analysis, GD defined in Equation \ref{eq:gd} is the \textit{Euler method} to approximately solve the initial value problem of the gradient flow in Equation \ref{eq:continuous}, and $\vtc \approx \vtd$ when $\eta$ is sufficiently small~\cite{continuous_v_s_discrete}. In Section \ref{sec:exp}, we show that the results derived from this limit align well with the experiments in discrete settings.

\section{Theory for Optimal Learning of LMs}
\label{sec:learning_law}
In this section, we present our theory in the continuous limit of GD. We first propose an objective for ``maximizing the reduction rate of $L^\dsr$ by optimizing the learning policy''. Then, we derive our main theorem, named \textit{Learning Law}, which introduces a necessary condition for the dynamics of the learning process induced by the policy that achieves the optimum of the objective.

\subsection{Objective: Maximizing Compression Ratio} 
We characterize the reduction rate of $L^\dsr$ with the area under the curve of $L^\dsr(\vtc)$ (AUC of $L^\dsr$) and minimize this area to achieve high learning speed:
\begin{equation}
\begin{aligned}
    \min_{\vga(t)} \ \ & \int_0^T L^\dsr(\vtheta_{\vga}(t)) \rmd t, \\
    \text{s.t.} \ \ & \sum_{n=1}^N \gnc = 1,\\
                      & \gnc \ge 0, n=1,2,\cdots,N,
\end{aligned}
\label{eq:obj}
\end{equation}
where $\vga(t)=\left[\gamma_1(t), \gamma_2(t), \cdots, \gnc\right]^\top$ and $\vtheta_{\vga}(t)$ is an alias of $\vtc$ satisfying Equation \ref{eq:continuous} to emphasize its dependency on $\vga(t)$. As shown in Figure \ref{fig:obj}, for sufficiently large $T$, a learning process with minimal loss AUC owns the highest loss reduction rate.
Interestingly, the AUC of $L^\dsr$ has a physical significance from the ``LM-training-as-lossless-compression'' view~\cite{jack_rae_compression}: \textbf{the resulting description length of compressing data drawn from the desired data distribution}.
Therefore, Equation \ref{eq:obj} is equivalent to maximizing the corresponding compression ratio. Note that unlike \cite{lm_is_compression} that studies encoding data using a well-trained LM, we view the entire LM training as a compression process. We provide more discussion of these two perspectives in Section \ref{sec:related_work}.
Besides, there are still slight differences between our statement and that in prior works viewing the training process as lossless compression~\cite{nncp,trm_text_compress,jack_rae_compression}: we consider \textit{the desired loss AUC of GD training for multiple epochs}, while the previous statement is about \textit{the training loss AUC with single-epoch SGD training}. More discussion about this difference can be found in Appendix \ref{app:our_compression}. 

\subsection{Learning Law}
\label{sec:derive}
Equation \ref{eq:obj} defines an Optimal Control problem that can be solved by \textit{Maximum Principle}~\cite{optimal_control}. However, we find the solution hard to interpret and verify in practical LM learning. Therefore, in this work, we derive a looser necessary condition for the optimum of Equation \ref{eq:obj}.
\begin{tcolorbox}
\begin{theorem}[Learning Law]
When an LM is trained with an optimal learning policy, which yields a learning process corresponding to a maximum compression ratio on the desired data distribution, the following condition holds for $0 < t \le T$ and any $m$, $n$ such that $\gamma_m(t) > 0$, $\gnc > 0$:
\begin{equation}
\begin{aligned}
    \nabla L \cdot \nabla l_m = \nabla L \cdot \nabla l_n = \Const,
    \label{eq:main}
\end{aligned}    
\end{equation}
where $\nabla L=\nabla L^{\mathrm{dsr}}(\vtc)=\nabla\frac{1}{K}\sum^{K}_{k=1} l(x^{\mathrm{dsr}}_k, \vtc)$, $\nabla l_m = \nabla l(x_m^{\mathrm{trn}}, \vtc)$, $\nabla l_n = \nabla l(x_n^{\mathrm{trn}}, \vtc)$, and $\cdot$ is dot-product. $\Const=-\frac{\rmd}{\rmd t}L^{\mathrm{dsr}}(\vtc) $ is the desired loss change rate over time and \textbf{is independent of $\bm{n}$ and $\bm{m}$}.
\label{trm:main}
\end{theorem}    
\end{tcolorbox}

To prove Theorem \ref{trm:main}, we apply the Euler-Lagrange (EL) equation~\cite{el_equation} and Karush–Kuhn–Tucker (KKT) conditions~\cite{kkt} to Equation \ref{eq:obj}, which results in the condition: $\nabla \ldsrc \cdot \nabla l(x_n^\trn, \vtc) = -\frac{\rmd}{\rmd t}\ldsrc$. A full proof is shown in Appendix \ref{app:derive}. 

$\nabla L \cdot \nabla l_n$ in Equation \ref{eq:main} represents the \textbf{contribution} of the training example $x_n^\trn$ to $\ldsrc$, which is maximized when the gradient on $x_n^\trn$ shares the same direction with the gradient of $\ldsrc$. We denote $\bm{\con_n(t) = \nabla L \cdot \nabla l_n = \nabla \ldsrc \cdot \nabla l(x_n^\trn, \vtc)}$ for convenience in the rest of the paper. Note that when the model is converged ($\nabla L^\trn(\vtc, t) \approx \bm{0}$), $\con_n(t)$ can be viewed as an approximation of the Influence Function~\cite{if} by setting the Hessian matrix of $L^\trn(\vtheta, t)$ at $\vtheta = \vtc$ to an identity matrix~\citep{tracin}. In essence, Equation \ref{eq:main} means $\con_n(t)$ equals a value independent of $n$. Since the zero-weight examples ($\gnc = 0$) are typically noisy (verified in Section \ref{sec:zero-weight}), Theorem \ref{trm:main} suggests that \textbf{all non-noisy examples should be identically contributive to the LM in the optimal learning process}. In the following, we provide more discussion of this theorem.

\subsection{Discussion}
\label{sec:disc}
\paragraph{Theorem \ref{trm:main} suggests a matching of the local and global learning.} Another interpretation of $\con_n(t)$ is the ``local learning speed'': how fast the LM learns the knowledge in $x_n^\trn$ that is helpful to reduce $L^\dsr$. This is because the dot-product operation in $\con_n(t)$ can be viewed as the projection of the individual loss descending velocity $\nabla l(x_n^\trn, \vtc)$ on the desired direction. Correspondingly, $\frac{\rmd}{\rmd t} \ldsrc$ represents the LM's ``global learning speed'': how fast the LM gets better by learning all individual $x_n^\trn$. As a result, $\con_n(t) = \Const = -\frac{\rmd}{\rmd t}\ldsrc$ in Theorem \ref{trm:main} indicates that the local learning speed should match the global learning speed in the optimal learning process. 

\paragraph{The optimal learning policy establishes a dynamic data re-weighting strategy.} Generally, as the learning of LM progresses, $\con_n(t)$ drops because the gradient norm on each example $||\nabla l(x_n^\trn, \vtc)||$ decreases as the LM fits $x_n^\trn$. In addition, the direction of $\nabla l(x_n^\trn, \vtc)$ diverges from $\nabla  \ldsrc$ due to the possible discrepancy between the distribution of $x_n^\trn$ and $x_k^\dsr$, which also contributes to the decrease of $\con_n(t)$. Therefore, Theorem \ref{trm:main} guarantees that \textit{highly contributive example $x_n^\trn$ with high $\con_n(t)$ obtains large weights for training}, in order to reduce $\con_n(t)$ to meet the value of other examples. On the other hand, Theorem \ref{trm:main} also ensures that \textit{the weights of $x_n^\trn$ are lowered before the LM over-fits it} because $\con_n(t)$ should not be too small to match the global learning speed. Altogether, this forms a dynamic training data re-weighting strategy, which is intuitively essential for the optimal learning policy that maximizes the learning speed of an LM.

\paragraph{Theorem \ref{trm:main} is a necessary condition for the optimal learning dynamics.} This is because the E-L equation and KKT conditions are necessary conditions for the global optimum when the optimization problem is non-convex. Therefore, a learning process satisfying Theorem \ref{trm:main} is not guaranteed optimal. For example, by setting $\gamma_1(t)=1$ and $\gamma_2(t)=\gamma_3(t)=\cdots=\gamma_N(t)=0$, Equation \ref{eq:main} is satisfied, regardless of the values of $\con_n(t)$. This learning policy corresponds to using SGD with mini-batch size = 1, which is unlikely to be the optimal~\cite{oai_bs_study}. Therefore, searching for the optimal policy according to Theorem \ref{trm:main} may need regularization terms in practice, which we leave for future work to explore.

\section{Experiments}
\label{sec:exp}

We conduct experiments in the discrete setting of Equation \ref{eq:gd}, where the conclusions derived from the continuous limits in Section \ref{sec:learning_law} are still applicable when $\eta$ is sufficiently small~\cite{continuous_v_s_discrete}. We first design a method to find the optimal learning policy $\vga_t \in \sR^{N}$ for $0 \le t \le T-1$, by explicitly minimizing the AUC of $\ldsrd$ in the discrete setting, which maximizes the corresponding compression ratio of data drawn from the desired distribution. Then we examine our Learning Law (Theorem \ref{trm:main}) on the learning process induced by the found policies. Finally, we empirically verify that maximizing the compression ratio essentially improves the scaling law coefficients~\cite{scaling_law}, indicating the practical significance and promise of our theory.

\subsection{Finding the Optimal Learning Policy}
\label{sec:method}
To find the optimal $\vga_t$, we directly solve the discrete version of the optimization problem defined in Equation \ref{eq:obj} with a Proximal Gradient Method~\cite{prox_gd}:
\begin{equation}
    \begin{aligned}
        J(\vga) &= \sum_{t=1}^T \ldsrd, \\
        \vga_t &\leftarrow \operatorname{Proj}\left[\vga_t - \epsilon \nabla_{\vga_t} J(\vga)\right], \ 0\le t \le T-1,
    \end{aligned}
    \label{eq:method}
\end{equation}
where $J(\vga)$ is a discrete approximation of the integral in Equation \ref{eq:obj},  $\epsilon$ is the learning rate and $\operatorname{Proj}[\cdot]$ projects a point in $\sR^N$ to the $N$-simplex, ensuring that $\vga_t$ is a probability distribution over $N$ training examples. The optimization process can be implemented efficiently using dynamic programming and Jacobian-Vector-Product in PyTorch~\cite{pytorch}, which is described in detail in Appendix \ref{app:policy_opt}.

\begin{figure*}[t]
	\centering
	\subfigure[Perceptron Linear Classification]{
		\includegraphics[width=0.46\textwidth]{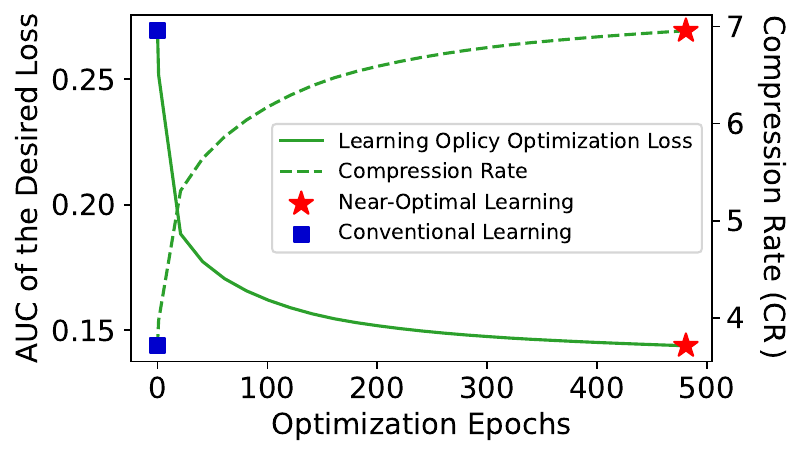} 
        \label{fig:exp_opt_pcn}
	}
 	\subfigure[Transformer Language Modeling]{
            \includegraphics[width=0.46\textwidth]{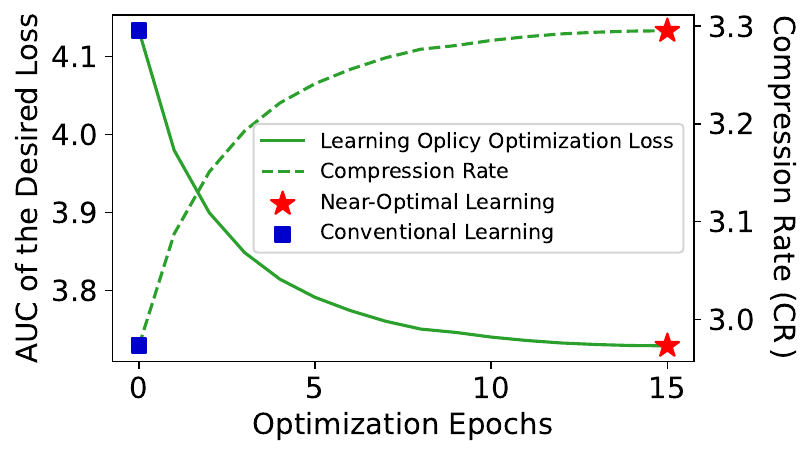} 
        \label{fig:exp_opt_trm}
	}
	\caption{Learning policy optimization results in Perceptron linear classification \textbf{(a)} and Transformer language modeling tasks \textbf{(b)}. We plot the learning policy optimization loss $J(\gamma)$ (solid lines), defined in Equation \ref{eq:method}, which represents the area under the curve (AUC) of the desired Perceptron or Transformer loss. We also show the corresponding compression ratio of the training process (dashed lines) in an "LM-as-Lossless-Compression" view. The optimization starts from conventional learning and smoothly converges to near-optimal learning with low loss AUC and high comprehension rate.}
    \label{fig:exp_opt}
\end{figure*}

\begin{figure*}[t]
	\centering
        \subfigure[Perceptron Linear Classification]{
   		 \includegraphics[width=0.46\textwidth]{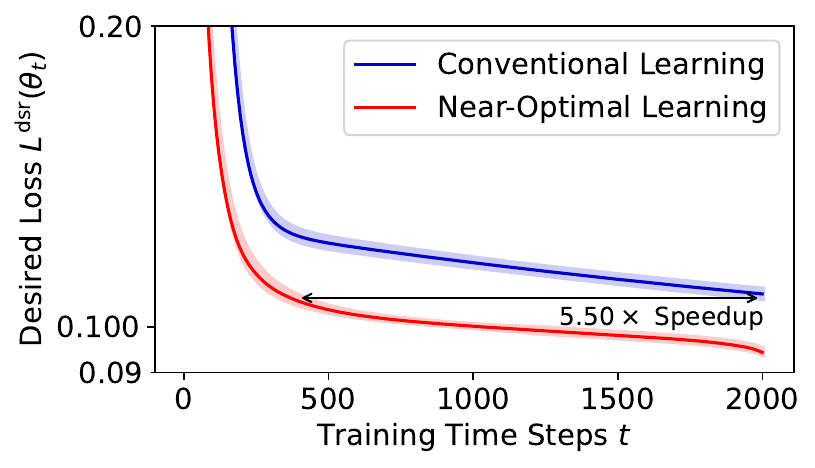}
        \label{fig:exp_train_pcn}
        }
        \subfigure[Transformer Language Modeling]{
		 \includegraphics[width=0.44\textwidth]{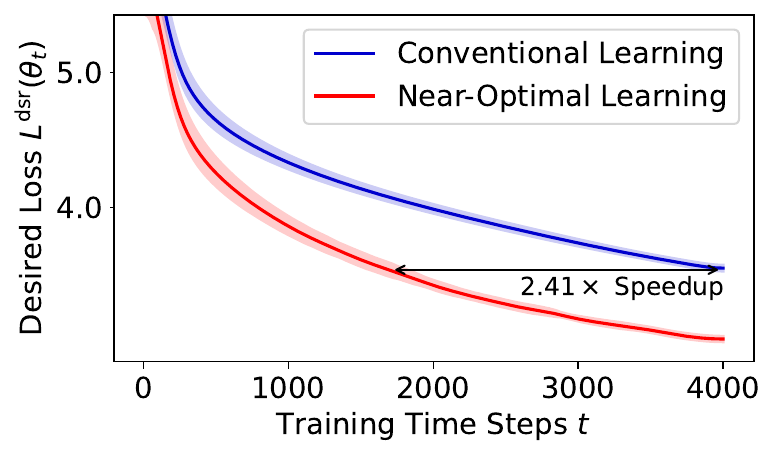}
        \label{fig:exp_train_trm}
        }
	\caption{Curves of the desired loss $\ldsrd$ when the model is trained using the conventional and the near-optimal learning policy. The near-optimal learning process achieves $5.50\times$ speedup in Perceptron linear classification \textbf{(a)} and $2.41\times$ speedup in Transformer language modeling \textbf{(b)}.}
    \label{fig:exp_train}
\end{figure*}

\subsection{Experimental Setup}
\label{sec:exp_setup}
We conduct experiments on a linear classification task based on Perceptron~\cite{perceptron} and a language modeling task based on Transformer~\cite{transformer}. See Appendix \ref{app:hp} for hyper-parameter configurations.

\paragraph{Perceptron Linear Classification.} We adopt a teacher-student setting~\cite{stat_mech} where each example $x_n=(\rvz_n, y_n)$ is a pair of $D$-dimensional vector $\rvz_n\in \sR^D$ drawn i.i.d. from Gaussian distribution, and a scalar $y_n = \operatorname{sign}(\rmT \cdot \rvz_n)$ given the ground truth weight $\rmT \in \sR^D$. We introduce a shift between the training and the desired data distribution to reflect their differences. 
The data are learned by an one-layer Perception parameterized by $\vtheta\in \sR^D$: $o_n=\sigmoid (\vtheta \cdot \rvz_n) = \frac{1}{1+\exp (-\vtheta \cdot \rvz_n)}$, which is trained with Maximum Likelihood Estimation (MLE) loss $l(x_n, \vtheta) = -\log o_n^{y_n}(1-o_n)^{1-y_n}$. In Appendix \ref{app:perceptron_as_comp}, we show that Perceptron can be viewed as a one-step LM, which means our theory still applies.

\paragraph{Transformer Language Modeling.} Considering the computation cost of the optimal policy searching, we adopt a two-layer Transformer with about 1.7M parameters and train it on TinyStories~\cite{tinystories}, a high-quality pre-training corpus. We add perturbations to the training examples (see Appendix \ref{app:hp} for details), which mimics the relatively low quality of the pre-training corpus in practice. Since our theoretical derivation is generally applicable, we believe that our theory also applies to larger LMs.

To migrate the risk of over-fitting the $K$ examples used to compute $\ldsrd$ in Section \ref{sec:method}, we additionally construct a held-out test set with $K$ examples from the desired data distribution in both Perceptron linear classification and Transformer language modeling experiments. In the following, we \textbf{compute and report the evaluation metrics by treating the test examples, unseen during the policy optimization, as $\bm{x_k^\dsr}$ in Equation \ref{eq:tg}}.

\subsection{Learning Policy Optimization Results}

\paragraph{A near-optimal learning policy can be found with the method in Section \ref{sec:method}.} In Figure \ref{fig:exp_opt}, we show the optimization process of finding the optimal learning policy. We plot the learning policy optimization loss $J(\vga)$, which is also the AUC of $\ldsrd$ in the learning process induced by $\vga_t$, and the corresponding compression ratio $\CR = \frac{T\log|V|}{\sum_{t=1}^T \ldsrd}$, where $V$ is the size of the label / vocabulary space for Perceptron / transformer (see Appendix \ref{app:orig_compression} for more explanation). The curve of $J(\vga)$ is smooth and almost converges at the end, indicating that a near-optimal learning policy is found. 

\paragraph{The near-optimal learning policy yields a high acceleration ratio of the learning speed.} In Figure \ref{fig:exp_train}, we plot the curve of $\ldsrd$ when the Perceptron and Transformer are trained under the conventional and near-optimal learning policies. The near-optimal policies significantly improve the loss AUC, bringing about acceleration $5.50\times$ and $2.41\times$ at the end of the Perceptron and Transformer training, respectively. Note that all reported metrics are computed on the test set unseen during the policy optimization, suggesting that the near-optimal policy does not over-fit the specific examples used to compute $\ldsrd$ but helps the model learn faster on the desired distribution.

\begin{figure*}[t]
	\centering
	\subfigure[Perceptron Linear Classification]{
		\begin{minipage}[b]{0.48\textwidth}
			\includegraphics[width=0.98\textwidth]{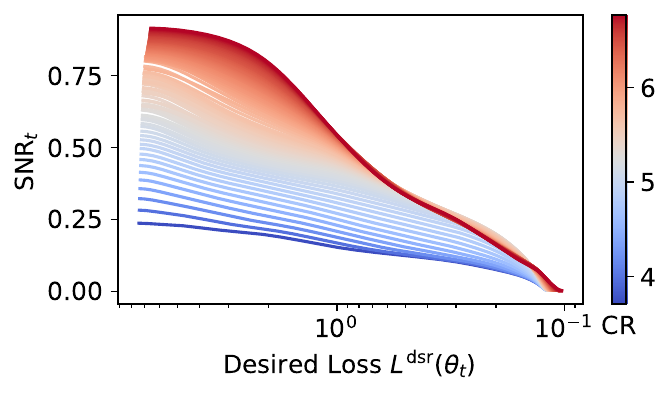} 
            \label{fig:exp_main_pcn}
		\end{minipage}
	}
    	\subfigure[Transformer Language Modeling]{
    		\begin{minipage}[b]{0.48\textwidth}
   		 	\includegraphics[width=0.98\textwidth]{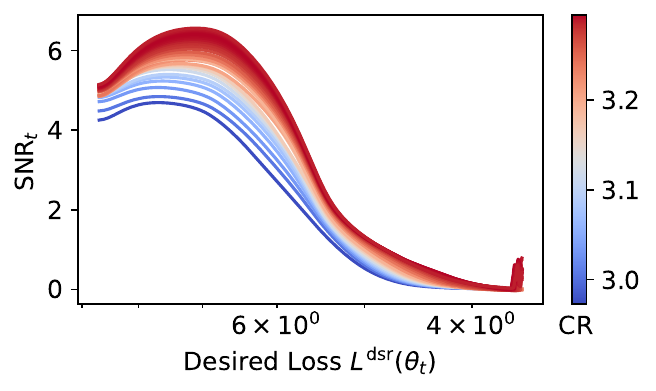}
            \label{fig:exp_main_trm}
    	\end{minipage}
    	}
        \caption{Empirical evidence of our Learning Law (Theorem \ref{trm:main}) in Perceptron linear classification \textbf{(a)} and Transformer language modeling \textbf{(b)} tasks. We measure the degree of similarity in contribution among different samples by $\SNR_t$, the \textit{Signal-Noise-Ratio} of the contribution $\con_{n,t}$ of training examples, calculated as the mean divided by the standard deviation of $\con_{n,t}$ across examples (Equation \ref{eq:snr}). Higher $\SNR_t$ means better contribution similarity. 
        We plot $\SNR_t$ with respect to the desired loss $\ldsrd$ under different learning processes. Each line is a certain learning process, whose color means the corresponding compression ratio ($\CR$). Runs with higher $\CR$ generally get higher $\SNR_t$ throughout learning, indicating that the example contributions are more similar to each other in a learning process closer to the optimum, which is in line with our Learning Law (Theorem \ref{trm:main}).}
        \label{fig:exp_main}
\end{figure*}

\begin{figure}[t]
	\centering
        \vspace{-0.2cm}
	\subfigure[Perceptron Linear Classification]{
		\begin{minipage}[b]{0.46\textwidth}
                \centering
			\includegraphics[width=0.70\textwidth]{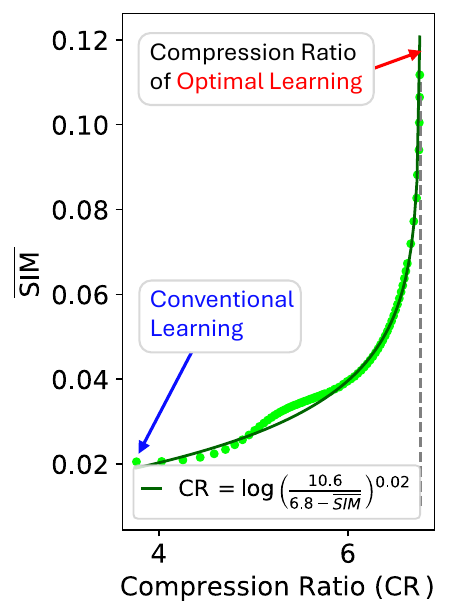} 
            \label{fig:exp_cr_snr_linear}
		\end{minipage}
	}
    \subfigure[Transformer Language Modeling]{
    		\begin{minipage}[b]{0.46\textwidth}
                \centering
   		 	\includegraphics[width=0.70\textwidth]{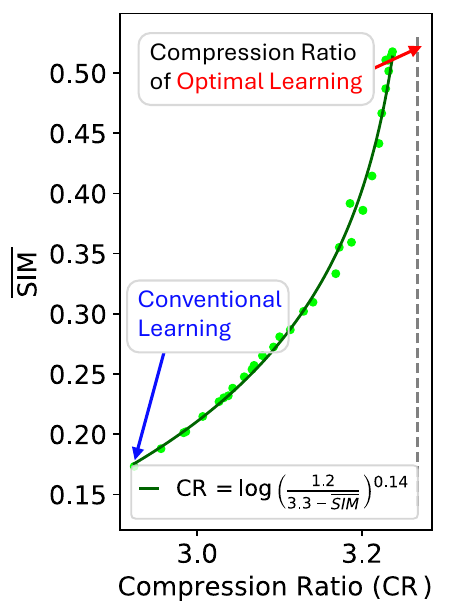}
            \label{fig:exp_cr_snt_trm}
    	\end{minipage}
    	}
        \caption{Empirical evidence of the Learning Law (Theorem \ref{trm:main}) in Perceptron linear classification \textbf{(a)} and Transformer language modeling \textbf{(b)} tasks. Following Figure \ref{fig:exp_main}, we consider $\overline{\SNR} = \frac{1}{T}\sum_{t=1}^T \SNR_t$, which summarizes the similarity of the training example contributions in a learning process.
        We plot the relationship between $\overline{\SNR}$ and $\CR$, and observe an evident tendency that $\overline{\SNR} \rightarrow +\infty$ when $\CR$ approaches a certain value, which can be fit by $\overline{\SNR}=\log \left(\frac{a}{b-\CR}\right)^c$. When the learning process approaches the optimum ($\CR\to b$), the standard deviations of training example contributions should be zero to allow $\overline{\SNR}\to +\infty$. This verifies Learning Law (Theorem \ref{trm:main}) that all training examples have the same contribution to the model in optimal learning.}
	\label{fig:exp_cr_snr}
    \vspace{-0.1cm}
\end{figure}

\subsection{Direct Verification of Learning Law (Theorem \ref{trm:main})}
\label{sec:exp_main}

We examine the similarity between $\con_{n,t}$ which is the discrete version of the individual sample contribution $\con_{n}(t)$ in a certain learning policy and satisfies $\con_{n,t}=\con_{n}(t)$ for $t=1,2,\cdots, T$. The similarity ($\SNR$) is measured by the \textit{Signal-Noise-Ratio} of $\con_{n,t}$:
\begin{equation}
    \label{eq:snr}
    \SNR_t = \frac{\overline{\con}_t}{s_{\con, t}},
\end{equation}
where $\overline{\con}_t = \sum_{n=1}^N\gnd \con_{n,t}$ is the weighted mean and $s_{\con,t} = \sqrt{\frac{\sum_{n=1}^N\mathbbm{1}\left[\gnd\ne 0\right]\left(\con_{n,t} - \overline{\con}_t\right)^2}{\sum_{n=1}^N\mathbbm{1}\left[\gnd\ne 0\right]-1}}$ is the standard deviation of $\con_{n,t}$ for training examples with non-zero weight.
The higher $\SNR_t$ means that the training examples have more similar $\con_{n,t}$. Note that $\SNR_t$ is dimensionless, which avoids the impact of the absolute value scale change of $\con_{n,t}$ during learning. We also consider $\overline{\SNR}= \frac{1}{T}\sum_{t=1}^T \SNR_t$, which summarizes the similarities of $\con_{n,t}$ throughout the learning process.

\paragraph{Higher compression ratio correlates with higher sample contribution similarities.} In Figure \ref{fig:exp_main}, we examine the value of $\SNR_t$ in the learning process induced by each policy found along the optimization process of $\vga_t$. Since the found policies bring about faster convergence, we plot $\SNR_t$ with respect to $\ldsrd$, rather than $t$. In this way, $\SNR_t$ are compared at the same ``stage'' of the model learning, migrating the impact of different convergence speeds. Figure \ref{fig:exp_main} demonstrates that the learning process with a higher compression ratio ($\CR$) generally keeps higher $\SNR_t$ in model learning, indicating that the contributions $\con_{n,t}$ of individual samples are more similar to each other throughout the learning process, which aligns with our Learning Law (Theorem \ref{trm:main}).


\paragraph{Sample contributions tend to be equal when the learning process approaches the optimum.} In Figure \ref{fig:exp_cr_snr}, we plot $\overline{\SNR}$ with respect to $\CR$ for each learning process. We observe an evident tendency that $\overline{\SNR} \rightarrow +\infty$ when $\CR$ approaches a certain value. Accordingly, we use the function $\overline{\SNR}=\log \left(\frac{a}{b-\CR}\right)^c$ to fit the tendency of the experimental observations.
Figure \ref{fig:exp_cr_snr} indicates that when the learning process continuously improves until the optimum ($\CR \rightarrow b$), the standard deviation of $\con_{n,t}$ should be zero to allow $\overline{\SNR} \rightarrow +\infty$. This verifies Learning Law (Theorem \ref{trm:main}) that the contributions of non-zero-weight training samples ($\con_{n,t}$) are identical in optimal learning.

\subsection{Properties of Zero-Weight Examples}
\label{sec:zero-weight}

\begin{figure*}[t]
	\centering
        \subfigure[Perceptron Linear Classification]{
		\begin{minipage}[b]{0.48\textwidth}
                \centering
			\includegraphics[width=0.86\textwidth]{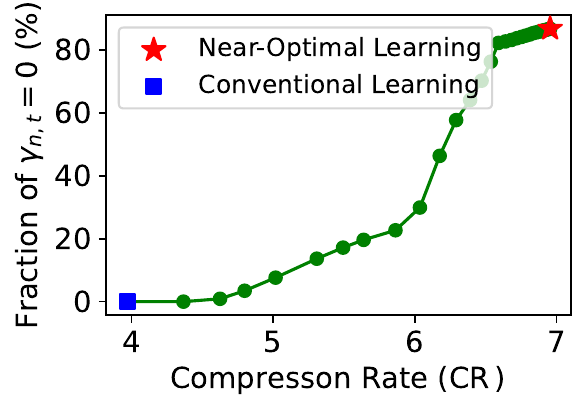}
            \label{fig:exp_prop1_pcn}
		\end{minipage}
	}
         \subfigure[Transformer Langnauge Modeling]{
            \begin{minipage}[b]{0.48\textwidth}
            \centering
            \includegraphics[width=0.86\textwidth]{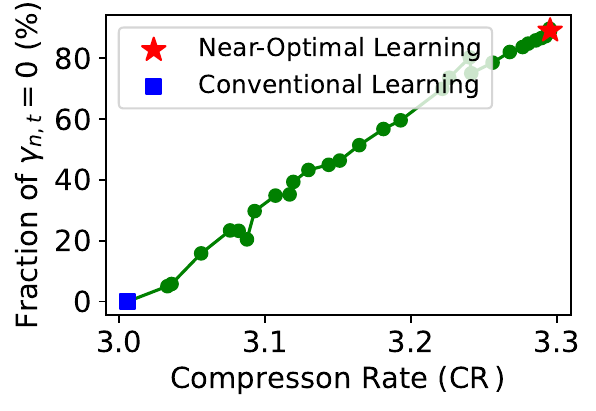}
            \label{fig:exp_prop1_trm}
            \end{minipage}
        }
        \caption{Empirical evidence of Property \ref{cor:noise}: non-contributive and noisy examples are excluded in optimal learning. The y-axis is the fraction of zero-weight examples among those with $\con_{n,t} \le 0$ at the same time step. Each point represents a learning policy, which tends to assign the example weight $\gnd=0$ to 100\% of noisy and non-contributive data when it approaches the optimum. }
        \label{fig:exp_prop1}
\end{figure*}

\begin{figure*}[t]
\centering
\includegraphics[width=0.5\textwidth]{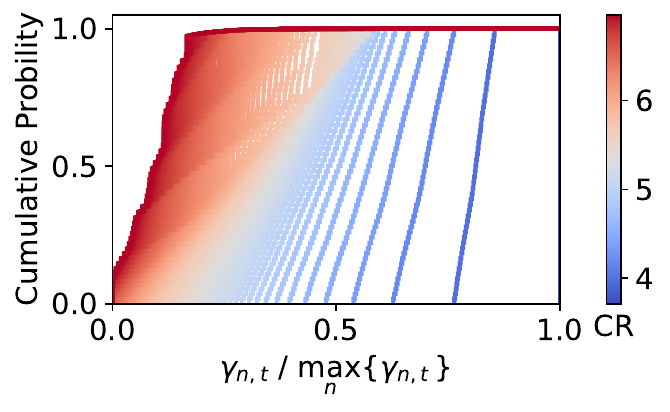}
\caption{Empirical evidence of Property \ref{cor:pft}: perfectly learned examples are ignored in optimal learning. We plot the cumulative distribution function (CDF) of the example weights $\gnd$ that satisfies $l(x_n^\trn, \vtd) < 1\times 10^{-6}$. Each line corresponds to a learning process. A large fraction of low-loss examples (perfectly learned) in the near-optimal learning obtain small $\gnd$ values (ignored), and this tendency becomes more evident when the learning approaches its optimum ($\CR$ increases).}
\label{fig:exp_prop2}
\vspace{-0.1cm}
\end{figure*}

\begin{figure*}[t]
	\centering
        \subfigure[Perceptron Linear Classification]{
		\begin{minipage}[b]{0.46\textwidth}
			\includegraphics[width=0.98\textwidth]{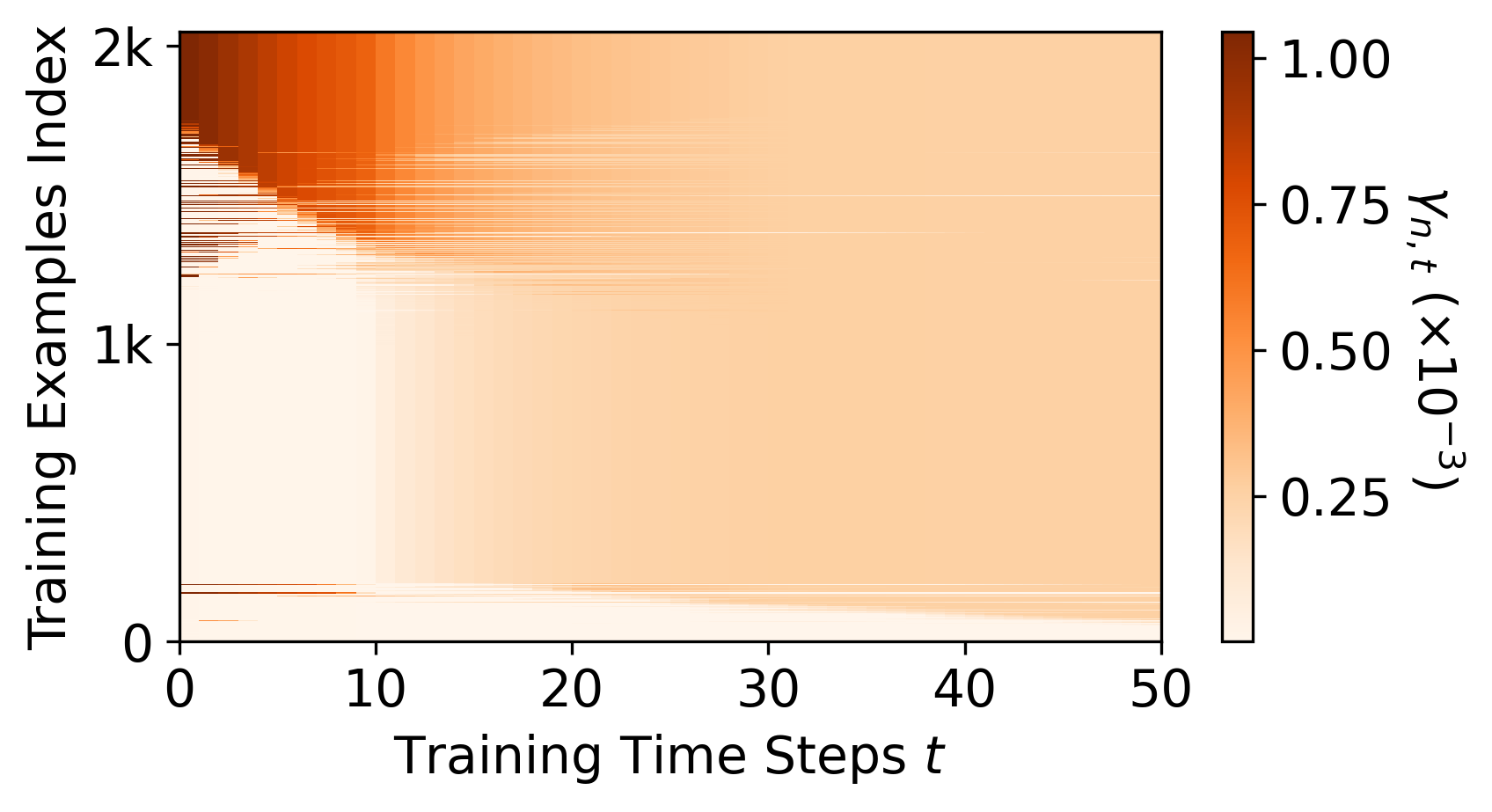}
            \label{fig:exp_prop3_pcn}
		\end{minipage}
	}
         \subfigure[Transformer Language Modeling]{
            \begin{minipage}[b]{0.46\textwidth}
            \includegraphics[width=0.98\textwidth]{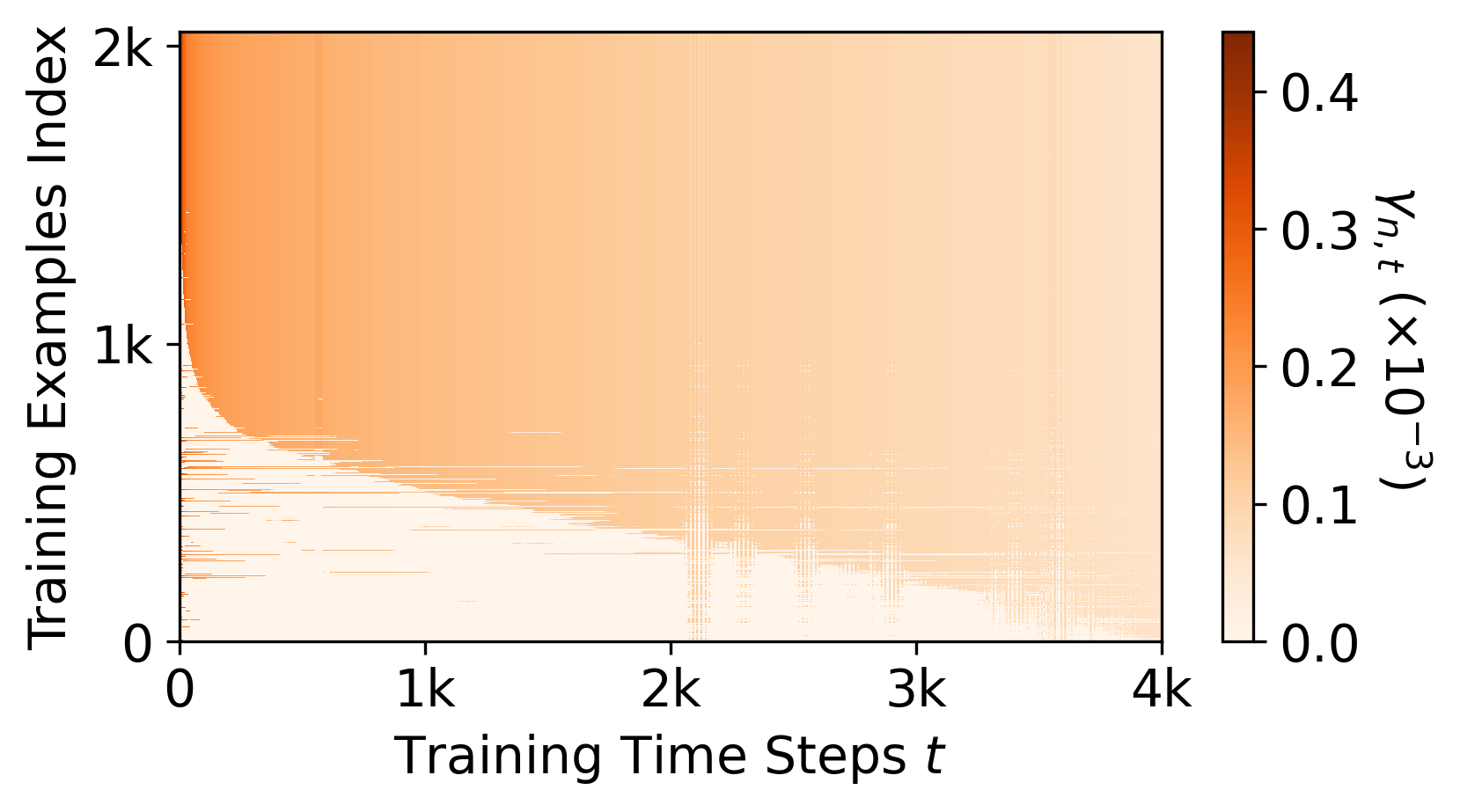}
            \label{fig:exp_prop3_trm}
            \end{minipage}
        }
        \caption{Empirical evidence of Property \ref{cor:redd}: redundant training examples are discarded in optimal learning. We randomly sample 2048 training examples satisfying $\con_{n,t}>0$ (contributive and unlearned examples) throughout the near-optimal learning process and show the dynamics of the example weight $\gnd$ (represented by the color in (a) and (b)). Since Perceptron converges quickly, we only plot its $\gnd$ dynamics for $t \le 50$. The near-optimal policies assign $\gnd=0$ to redundant examples in addition to the perfectly learned and non-contributive data points.
        }
        \label{fig:exp_prop3}
\end{figure*}

The experiments in Section \ref{sec:exp_main} mostly focus on the non-zero-weight examples. In this section, we provide more empirical evidence for Learning Law (Theorem \ref{trm:main}) by examining the properties of the examples with $\gnd=0$. We derive three properties of the optimal learning dynamics from Theorem \ref{trm:main} and then verify them through experiments. \textbf{The first property} guarantees that examples with non-positive contributions receive $\gnd=0$, indicating that the ``noisy'' examples at each time step are excluded by the optimal learning policy:
\begin{property}
    \label{cor:noise}
     The training example $x_n^\trn$ whose $\con_{n,t} \le 0$ gets $\gnd=0$ before the model converges.
\end{property}
\begin{proof}
    Before convergence, $\frac{\rmd \ldsrc}{\rmd t} < 0$ holds, indicating $\con_{n,t} > 0$ for $x_n^\trn$ that satisfies $\gnd>0$, according to Theorem \ref{trm:main}. Therefore, $\con_{n,t} \le 0 \Rightarrow \gnd=0$.
\end{proof}
\textit{Empirical Evidence.} We calculate the fraction of zero-weight examples ($\gnd=0$) among all examples with non-positive contributions at $t$ ($\con_{n,t} \le 0$): $\frac{\sum_{n,t} \mathbbm{1}[\gnd=0]\mathbbm{1}[\con_{n,t} \le 0]}{\sum_{n,t} \mathbbm{1}[\con_{n,t} \le 0]}$ and plot this fraction with respect to the $\CR$ value of the corresponding learning process in Figure \ref{fig:exp_prop1}.
We can see that when the learning process approaches the optimum, the fraction tends to 100\%, indicating that the non-contributive examples are discarded.

\vspace{0.2cm}

\textbf{The second property} is derived only for Perceptron linear classification, which indicates that the optimal learning policy will ignore those perfectly learned training examples:
\begin{property}
For Perceptrons, the perfectly learned $x^{\mathrm{trn}}_n$, whose margin $(2y^{\mathrm{trn}}_n-1)\vtd \cdot \rvz^{\mathrm{trn}}_n \rightarrow +\infty$ at the time step $t$, gets $\gnd = 0$ in the optimal learning policy when the model is yet converged.
\label{cor:pft}
\end{property}
\vspace{-0.3cm}
\begin{proof}
When $(2y^\trn_n-1)\vtheta_t \cdot \rvz^\trn_n \rightarrow +\infty$, we have $o^\trn_n-y^\trn_n \rightarrow 0$, which means $\nabla l(x^\trn_n, \vtheta_t) = (o^\trn_n-y^\trn_n)\rvz^\trn_n \rightarrow \bm{0}$ and $\con_{n,t} \rightarrow 0$. 
Assuming $\gnd \ne 0$, according to Theorem \ref{trm:main}, we have $\con_{n,t} =-\frac{\rmd}{\rmd t}\ldsrc$ in the optimal learning process, which means that $\left|\frac{\rmd}{\rmd t}\ldsrc\right|$ should be arbitrarily small. This does not hold when the model is not converged. Therefore, we have $\gnd = 0$.
\end{proof}
\vspace{-0.2cm}
\textit{Empirical Evidence.} In Figure \ref{fig:exp_prop2}, we plot the cumulative probability distribution function of $\frac{\gnd}{\max_n\left\{\gnd\right\}}$ for the well-learned Perceptron training examples $x^\trn_n$ with near-zero per-instance training loss: $l(x_n^\trn, \vtheta) < 1\times 10^{-6}$. Figure \ref{fig:exp_prop2} shows that for the near-optimal policy, more than 90\% well-learned examples have relatively low $\gnd$ (< 0.2 $\max_n\left\{\gnd\right\}$). 
This trend becomes more evident as the learning policy approaches the optimum ($\CR$ increases), which verifies Property \ref{cor:pft}. 

\vspace{0.4cm}

\textbf{The third property} suggests that the optimal learning policy will discard the ``redundant'' training examples. Although this property is derived from Perceptron linear classification, we empirically find that it also applies to Transformer language modeling. We call a set $\{x_n\}_{n=1}^N$ has ``redundant'' examples when the example inputs in the set are linearly correlated, i.e., there exist $K$ scalars $\{\alpha_n\}_{n=1}^N$, not all zero, such that $\sum_{n=1}^N\alpha_n\rvz_n=\bm{0}$.
\begin{property}
For Perceptrons, if the training set $\{x^{\mathrm{trn}}_{n}\}_{n=1}^N$ has redundant examples, with probability 1, at least one example $x^{\mathrm{trn}}_i$ gets $\gamma_{i,t} = 0$ at the time step $t$ when the model is yet converged in the optimal learning process. 
\label{cor:redd}
\end{property}
\vspace{-0.2cm}
\begin{proof}
Given that $\{x^\trn_n\}_{n=1}^N$ has redundant examples, there exist scalars $\{\alpha_n\}_{n=1}^N$, not all zero, such that $\sum_{n=1}^N\alpha_n\rvz^\trn_{n}=\bm{0}$, which means $\sum_{n=1}^N\frac{\alpha_n}{o^\trn_n-y^\trn_n}\con_{n,t}=0$.
Assuming $\forall 1\le n \le N$, $\gamma_{n,t} \ne 0$, according to Theorem \ref{trm:main}, we have $\con_{n,t}=-\frac{\rmd}{\rmd t}\ldsrc$, suggesting $\left(\sum_{n=1}^N\frac{\alpha_n}{o^\trn_n-y^\trn_n}\right)\frac{\rmd}{\rmd t}\ldsrc=0$. For i.i.d. inputs $\{\rvz^\trn_n\}_{n=1}^N$, with probability 1, $\sum_{n=1}^N\frac{\alpha_n}{o^\trn_n-y^\trn_n}\ne0$, which means $\frac{\rmd}{\rmd t}\ldsrc=0$. This does not hold when the model is yet converged. Therefore, we have the property that $\exists 1\le n_0 \le N, \text{such that } \gamma_{n_0,t} = 0$.
\end{proof}
\vspace{-0.2cm}
\textit{Empirical Evidence.} In Figure \ref{fig:exp_prop3}, we visualize the dynamics of the $\gnd$ value satisfying $\con_{n,t} > 0$ throughout the learning process of Perceptron and Transformer. For Perceptron, the model dimension (128) is lower than the number of training examples (4096), which means the training dataset is redundant. Figure \ref{fig:exp_prop3_pcn} shows that, \textit{given the absence of the non-contributive examples}, a large fraction of $\vga_t$ still receives relatively small values before the model converges, which is caused by the redundancy of the training set.
In Figure \ref{fig:exp_prop3_trm}, we observe a similar phenomenon for Transformer, although the dimension of $\vtd$ is larger than the number of training instances. We suspect the reason is that the intrinsic dimension of Transformer is usually much smaller than the dimension of $\vtd$~\cite{intrinsic_dimension}, which leads to the redundancy of the training set. 

\subsection{Essence of Learning Acceleration}
\label{sec:scaling_law}
We investigate the essential improvement brought by the near-optimal learning policy in the perspective of the scaling laws of LMs~\cite{scaling_law}, which reveals a power law between the increase of training steps and the reduction of the test loss ($\ldsrd$) after a warming-up stage $t_0$\footnote{This requires the batch size to be sufficiently large~\cite{scaling_law}, which is satisfied in our experiments.}:
\begin{equation}
    \begin{aligned}
    \ldsrd = L_0 + {\left(\frac{B}{t}\right)}^{\beta}, \ t > t_0,
    \end{aligned}
    \label{eq:scaling_law}
\end{equation}
where $(B, \beta)$ are scaling law coefficients. $L_0$ contains the information of the model-size scaling and irreducible loss, and is assumed to be unaffected by the learning policy. In the following, we study the scaling properties of conventional and near-optimal learning processes.
\begin{figure}[t]
	\centering
	\includegraphics[width=0.7\textwidth]{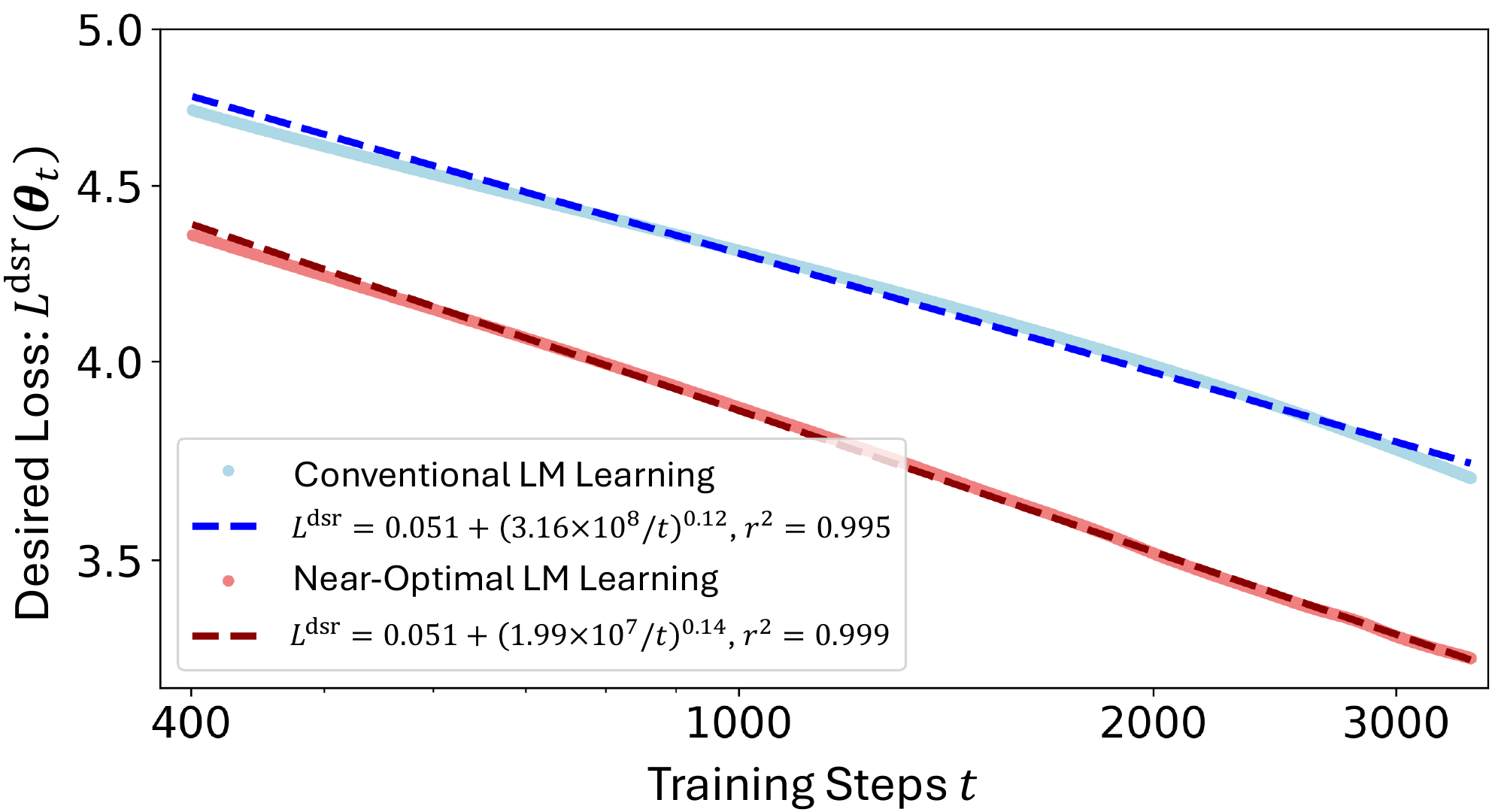}
            \caption{Illustration of the scaling law~\cite{scaling_law}: $\ldsrd = L_0 + (B/t)^\beta$ for conventional and near-optimal LM learning in Transformer language modeling. We fit the loss curves by the scaling law to obtain the correlation coefficient $r^2$ and show the loss curve (solid lines) together with the fit curve (dashed lines) in a log-log plot. The scaling law fits well for both conventional and near-optimal LM learning. The near-optimal LM learning essentially improves the coefficients $(B, \beta)$ in the scaling law by 96.6\% and 21.2\%, which shows great potential for speedup in training LLMs.}
    	\label{fig:exp_scaling}
\end{figure}
\begin{table}[t]
        \centering
        \begin{tabular}{rr|ccc}
        \toprule                
         $T$ & $N$   & $|\frac{\Delta B}{B}|$ (\%) & $|\frac{\Delta \beta}{\beta}|$ (\%) & $\mathrm{AR}$ \\ \midrule
         1K  & $2^{12}$  & 88.5          &  10.0                &  2.16                      \\
         2K  & $2^{13}$  & 94.9          &  18.0                &  2.31                      \\
         4K  & $2^{14}$ &  93.7          &  18.7                &  2.41                      \\
         8K  & $2^{15}$ &  94.8          &  19.0                &  2.48                      \\
        \bottomrule
        \end{tabular}
        \vspace{0.4cm}
        \caption{The improvements of the scaling law coefficients brought by the near-optimal learning policy for different total training steps ($T$) and data sizes ($N$) in Transformer language modeling. The vocabulary size increases with the growth of $N$ (see Appendix \ref{app:hp} for details). $\mathrm{AR}$ stands for the acceleration ratio as defined in Equation \ref{eq:ar}. The improvements hold for larger $T$ and $N$.}
        \label{tab:scaling_law}
\end{table}
\vspace{-0.1cm}
\paragraph{The near-optimal learning policy improves the scaling law coefficients of LMs.} 
In Figure \ref{fig:exp_scaling}, we fit the Transformer's loss curves induced by the conventional and near-optimal learning policies with Equation \ref{eq:scaling_law} by setting $t_0=400$ and $L_0=0.051$\footnote{In practice, we convert Equation \ref{eq:scaling_law} to $\ln (\ldsrd-L_0)=-\beta \ln t + \beta\ln B$, $(t > t_0)$ and perform linear regression. We search for $t_0$ and $L_0$ to get the highest correlation coefficients.}.
We observe that the near-optimal learning process still follows the scaling law, with $B$ and $\beta$ improved by 96.6\% and 21.2\% respectively.  Additionally, Table \ref{tab:scaling_law} shows that the improvement holds for the near-optimal policies found in the setting of larger $T$ and $N$.
We let $N$ grow with $T$ to ensure the sufficiency of training data ~\cite{chinchila}. The improvement of scaling law coefficients, especially $\beta$, provides significant potential in boosting the speed of LLM learning by taking advantage of power law growth. 
For two learning policies $\vga^{(1)}$ and $\vga^{(2)}$ which induce two loss curves $L^\dsr_{\vga^{(1)}}(\vtheta_{t})$ and $L^\dsr_{\vga^{(2)}}(\vtheta_{t})$ with two sets of scaling law coefficients $(B_1, \beta_1)$ and $(B_2, \beta_2)$, the acceleration ratio of $\vga^{(2)}$ over $\vga^{(1)}$ is:
\begin{equation}
    \begin{aligned}
    \mathrm{AR} = \frac{T}{\arg \min\limits_t \left\{L^\dsr_{\vga^{(2)}}(\vtheta_{t})\le L^\dsr_{\vga^{(1)}}(\vtheta_T)\right\}} = \frac{B_1^{\frac{\beta_1}{\beta_2}}}{B_2}T^{1-\frac{\beta_1}{\beta_2}}.
    \end{aligned}
    \label{eq:ar}
\end{equation}
For an LM pre-trained for 10M steps, we will obtain more than $9 \times$ acceleration at the end of the training if the scaling property of the LM is improved as in Figure \ref{fig:exp_scaling} and Table \ref{tab:scaling_law}.
Based on the recent experience in training LLMs~\cite{llama,llama2}, models are far from fully converged under the current training budget, which means small models (like 7B) have the potential to reach the performance of large models (like 65B), given enough training steps. 
However, according to Chinchilla's law~\cite{chinchila}, extending the training steps requires more computation than enlarging the model to achieve a certain performance. 
Therefore, by optimizing the learning policy to improve learning speed, the cost of training well-performed small models can be largely reduced, which is beneficial both for open-source endeavors in the LM research community and for the efficiency of industrial products. This indicates the promise and significance of designing practical learning policy optimization approaches, and our theory can be a valuable guide.

\section{Related Work}
\label{sec:related_work}

\paragraph{Improving the Learning Speed of Language Model.} There is a broad range of works that propose approaches to accelerate LM learning speed such as modifying model architectures~\cite{ln_study,layer_drop} or optimizers~\cite{lamb,sophia,grad_clipping}. There are also works studying the pre-training data programming to speed up LM convergence, such as data de-duplication~\cite{d4,semdedup}, domain mixture~\cite{doremi}, intrinsic task discovery~\cite{ppt}, and online data selection or re-ordering~\cite{skill-it,bilevel_training_dist,exp3_dp}, which can be viewed as special cases of optimizing learning policy. Unlike these works, we investigate the principles of optimizing LM learning in this paper.

\vspace{-0.2cm}

\paragraph{Language Modeling and Lossless Compression.} The recent success of LLMs calls for new interpretations beyond classical statistic learning theory for the fact that larger model sizes constantly cause better downstream generalization~\cite{double-descent,emergent}. One of the interpretations is to view the next-token-prediction training process of an LM as lossless data compression~\cite{nncp,trm_text_compress,jack_rae_compression}.
In this perspective, larger LMs have higher compression ratios, corresponding to better modeling of data generation regularities. It is worth noting that some recent works~\cite{llmzip,lm_is_compression} explore using well-trained LMs as compressors and thus the model sizes should be counted into the compressed data.
Unlike these works, viewing LM training as compression does not require including the model parameters in the compressed data (see Appendix \ref{app:orig_compression} for a constructive proof) and thus is more compatible with the model size scaling law of LMs~\cite{scaling_law}.


\vspace{-0.1cm}

\section{Discussion and Conclusion}
\paragraph{Summary.} In this work, we establish a theory for the optimal learning of LMs. We propose an objective that maximizes the compression ratio in an LM-training-as-losses-compression view. Then we derive a theorem, named \textit{Learning Law}, suggesting that all examples should be equally contributive to the LM in the optimal learning process, which is then validated by experiments in linear classification and real-world language modeling tasks. Finally, we empirically show that the optimal learning process essentially improves the scaling law coefficients of LMs, which sheds light on future works that design practical learning acceleration approaches.

\paragraph{Limitations.} One limitation of our work is that the experiments are conducted on relatively small scales. This is because our method to find the near-optimal learning policy corresponds to training a neural network with $L\times T$ layers, where $L$ is the layers of the LM and $T$ is the LM's total training steps (see Appendix \ref{app:policy_opt} for details). This leads to a high computational overhead when $L$ and $T$ scale up. However, since the theoretical derivation is generally applicable, we believe that our theory can be applied to LLMs. Another limitation is that our derivation assumes the LM is trained with full-batch GD, rather than some more commonly used techniques like mini-batch Adam~\cite{adam}. Since these methods are essentially gradient-based, our theory can still offer insights to future LM learning acceleration studies based on these techniques~\cite{tracin,tracein_lms}.

\paragraph{Future Work. } We believe that an important direction of future work is designing practical methods to find the optimal learning policies based on our theory for the large-scale training of LMs. Indeed, there are non-negligible challenges in this direction. Since the learning law provides a necessary condition for the learning policy's optimality, more regularization conditions may be required to prevent sub-optimal solutions. In addition, the approach to finding the optimal learning policy should be efficient enough without contributing much to the overall computation cost. Nevertheless, our work demonstrates the promise and potential of this direction. According to recent works on LLMs training~\cite{llama,llama2,mistral}, the losses are still far from convergence, which means that small models have the potential to reach the similar performance as large models, but are hindered by the computation overhead brought by the large total training steps. The optimal learning policy potentially brings about a large acceleration of training with the help of the power-law growth in Equation \ref{eq:ar}, which makes it possible to explore the limits of LMs given (inevitably) constrained computation and train a well-performed small LM that replaces current LLMs in practice.


\bibliography{law}
\bibliographystyle{alpha}


\clearpage
\appendix

\renewcommand{\algorithmicrequire}{\textbf{Input:}}
\renewcommand{\algorithmicensure}{\textbf{Output:}}

\section{Discussion of LM Training as Lossless Compression}
\subsection{Original View: Compressing the Training Data.} 
\label{app:orig_compression}

The idea of using an LM to compress data originates from the literature in the lossless text compression field~\cite{nncp,trm_text_compress}, and is recently adopted to interpret the essence of the next-token-prediction-based pre-training of LMs~\cite{jack_rae_compression}. We restate its core spirit by the following Theorem and the constructive proof from \cite{jack_rae_compression}:
\begin{theorem}
    \label{trm:compression}
    Consider an LM trained on a text corpus $\mathcal{D}$ with $M$ tokens using mini-batch next-token-prediction for one epoch. Let $B$ be the number of tokens in a batch and $L_t$ be the batch-averaged training loss at the time step $t$. Assume that $M$ is divisible by $B$. The training process can be viewed as lossless compression of the training data. The description length of the compressed data $\mathcal{C}$ is
    \begin{equation}
        \label{eq:desc_length}
        d(\mathcal{C}) = \sum_{t=1}^{M/B} B\cdot L_t + d(\mathrm{LM}),
    \end{equation}
    where $d(\mathrm{LM})$ is the length of the necessary code represented by a 0-1 string to run the LM training.
\end{theorem}
\begin{proof}
    The basic idea of the proof is to construct a lossless encoding and decoding process for $\mathcal{D}$ with the LM. Let $p_{\vtd}(\cdot | w_{<m})$ be the output distribution of the LM parameterized by $\vtd$ at the time step $t$, conditioning on the token prefix $w_{<m}=[w_{m-1}, w_{m-2}, \cdots, w_1]$. For simplicity, we assume that the LM is trained using mini-batch Stochastic Gradient Decent (SGD) with a learning rate $\eta$, where each batch is linearized to a continuous list of tokens. The batch-averaged training loss is $L_t=-\frac{1}{B}\sum_{m=1}^B \log p_{\vtd}(w_m | w_{<m})$\footnote{$\log (\cdot)$ stands for $\log_2(\cdot)$ in the following sections.}. The encoding the decoding process are described in Algorithm \ref{alg:enc} and \ref{alg:dec}.
    Basically, the main body of the algorithms other than the blue-colored parts implements the LM training.
    For encoding, the description length of a token $w_m$ is $-\log p_{\vtd}(w_m | w_{<m})$ according to Arithmetic Coding
    \footnote{\url{https://en.wikipedia.org/wiki/Arithmetic_coding}}
    , and thus the compressed length of a batch $\mathcal{W} = \{w_m\}_{m=1}^B$ is $\sum_{m=1}^B \left[-\log p_{\vtd}(w_m | w_{<m})\right]=B\cdot L_t$. $d(\mathcal{C})$ equals the sum of per-batch description lengths throughout the training plus the length of the code for LM training.
    Therefore, we get $d(\mathcal{C}) = B\cdot\sum_{t=1}^{M/B} \cdot L_t + d(\mathrm{LM})$.
    For decoding, since the code for LM training is the same as that in encoding, we have $\vtheta'_1=\vtheta_1$, and thus $w'_m=w_m$ for any steps in Algorithm $\ref{alg:dec}$, which can be easily proved by mathematical induction. As a result, $\mathcal{D}$ can be completely reconstructed from $\mathcal{C}$, indicating the encoding (compression) is lossless. 
\end{proof}

\begin{remark}
    The description length of the compressed data $d({\mathcal{C}})$ is approximately proportional to the area under the training loss curve ($\mathrm{AUC}$) when $M \gg 1$ because the size of LM training codes is much smaller than that of the compressed corpus and thus $d({\mathcal{C}})\approx \sum_{t=1}^{M/B} B\cdot L_t=B\cdot \mathrm{AUC}$.
\end{remark}

\begin{remark}
    Let $V$ be the vocabulary size of the LM and assume $M \gg 1$. The corresponding compression ratio of the learning process in Theorem \ref{trm:compression} is $\CR = \frac{M\log V}{d(\mathcal{C})} \approx \frac{M\log V}{\sum_{t=1}^{M/B} B\cdot L_t} \propto \frac{1}{\mathrm{AUC}}$. As the LM fits the data, we generally have $L_t < \log V$ because $\log V$ is the loss for a randomly initialized LM. This means the compression is valid, resulting in a compression ratio $\CR > 1$.
\end{remark}

Altogether, Theorem \ref{trm:compression} bridges a connection between data compression and LM training. Generally, a higher compression ratio indicates that the compression algorithm models the underlying data knowledge better and corresponds to a better performed LM, as stated in the following remark:
\begin{remark}
    \label{cor:compression}
    The LM's ability to model the knowledge in data is characterized by the corresponding lossless compression ratio of its learning process, which is inversely proportional to the loss AUC.
\end{remark}
Note that the model parameters are not included in the calculation of $d(\mathcal{C})$, and enlarging the model sizes typically reduces the loss AUC, which explains the remarkable performance of LLMs. In addition, $d(\mathcal{C})$ relates to the whole LM training process, not just the final loss. This is in line with the fact that larger LMs tend to perform better than smaller models, even if their final losses are the same~\cite{same_pt_better_down}. This observation supports the perspective that LM training can be conceptualized as a process of lossless data compression.

\begin{figure*}[t]

\begin{minipage}[t]{0.47\textwidth}
    \begin{algorithm}[H]
    \begin{algorithmic}
    \REQUIRE{Training corpus $\mathcal{D}$}
    \REQUIRE{The code for LM training as a 0-1 string}
    \ENSURE{Compressed data $\mathcal{C}$: list of 0-1 strings}
    \STATE Initialize $\mathcal{C}$ to an empty list
    \STATE Append the LM training code to $\mathcal{C}$
    \STATE Initialize the LM's parameters to $\vtheta_1$
    \FOR{$t \gets 1$ to $M/B$}
        \STATE Get a batch of tokens $\mathcal{W} = \{w_m\}_{m=1}^B$ from the training corpus $\mathcal{D}$
        \FOR{$m \gets 1$ to $B$, $w_m\in \mathcal{W}$}
            \STATE {\color{blue} Encode $w_m$ to a 0-1 string $s$ with Arithmetic Coding based on $p_{\theta_t}(\cdot | w_{<m})$}
            \vspace{0.09cm}
            \STATE {\color{blue} Append the 0-1 string $s$ to $\mathcal{C}$}
        \ENDFOR
        \STATE $L_t \gets -\frac{1}{B}\sum_{m=1}^B \log p_{\vtd}(w_m | w_{<m})$
        \STATE $\vtheta_{t+1} \gets \vtheta_{t} - \eta \nabla L_t$
    \ENDFOR
    \end{algorithmic}
    \caption{Encoding}
    \label{alg:enc}
\end{algorithm}
\end{minipage}
\hfill
\begin{minipage}[t]{0.52\textwidth}
    \vspace{0pt}
    \begin{algorithm}[H]
    \begin{algorithmic}
    \REQUIRE{Compressed data $\mathcal{C}$: list of 0-1 string}
    \ENSURE{Training corpus $\mathcal{D}$}
    \STATE Get the LM training code from the first string in $\mathcal{C}$
    \STATE Pop the first string from $\mathcal{C}$
    \STATE Initialize $\mathcal{D}$ to an empty list
    \STATE Initialize the LM's parameters to $\vtheta'_1$
    \FOR{$t \gets 1$ to $M/B$}
        \STATE Get a batch of 0-1 strings $\mathcal{S} = \{s_m\}_{m=1}^B$ from the compressed data $\mathcal{C}$
        \FOR{$k \gets 1$ to $B$, $s_m \in \mathcal{S}$}
            \STATE {\color{blue} Decode $w'_m$ from $s_m$ with Arithmetic Coding based on $p_{\theta'_t}(\cdot | w'_{<m})$}
            \STATE {\color{blue} Append the token $w'_m$ to $\mathcal{D}$}
        \ENDFOR
        \STATE $L_t \gets -\frac{1}{B}\sum_{m=1}^B \log p_{\vtheta'_t}(w'_m | w'_{<m})$
        \STATE $\vtheta'_{t+1} \gets \vtheta'_{t} - \eta \nabla L_t$
    \ENDFOR
    \end{algorithmic}
    \caption{Decoding}
    \label{alg:dec}
\end{algorithm}
\end{minipage}

\end{figure*}

\subsection{Our View: Compressing Data from the Desired Distribution.}
\label{app:our_compression}
Although we also focus on the loss AUC throughout the paper, our setting differs from that in Appendix \ref{app:orig_compression}:
(1) we assume the LM is trained with full-batch Gradient Descent (GD) for multiple epochs while Theorem \ref{trm:compression} lies in the scenario where the LM is trained with SGD for only one epoch;
(2) we consider $L^{\text{dsr}}$ computed on data other than the training examples, while $p_{\vtd}$ in Equation \ref{eq:desc_length} is computed on the training data.
However, although not entirely rigorous, we argue that Remark \ref{cor:compression} still holds despite the differences in (1) and (2).
The reason is that: \textbf{regarding (1)}, mini-batch SGD is an approximation of GD, which means they share the similar training dynamics when the batch size of SGD is large enough;
\textbf{regarding (2)}, just like the training loss AUC, the AUC of $L^{\text{dsr}}$ can be viewed as the description length of \textit{compressing examples from the desired data distribution} during the learning process. 
In this way, Remark \ref{cor:compression} indicates that minimizing the AUC of $L^{\text{dsr}}$ corresponds to optimizing the data compression on the desired distribution, which improves the LM's ability to model the desired data knowledge. 
This is more of practical concern because in most scenarios, the model performance is measured on a dataset other than the training set, such as the validation set in classical machine learning~\cite{nature_of_stat_learning}, the high-quality held-out corpus in large-scaling pre-training~\cite{scaling_law}, or the target set in domain adaption~\cite{data-selection-IS}.

\subsection{Perceptron Training as Lossless Compression}
Viewing model training as lossless compression stems from the next-token-prediction learning paradigm of LMs. 
We show that this perspective also fits in the one-epoch Maximum Likelihood Estimation (MLE) training of Perceptrons on the linear classification task, where the \textit{label of each example is compressed given the input vectors}. Specifically, the proof in Appendix \ref{app:orig_compression} still applies if we treat linear classification as a one-step language modeling with vocabulary size $V=2$.
Following the notation in Section \ref{sec:exp_setup}, for a Perceptron parameterized by $\vtd$ at the time step $t$, its probability of outputting $y$ conditioning on $\rvz$ is $p_{\vtd}(y|\rvz) = o^y(1-o)^{1-y}$, where $o=\sigmoid(\vtd \cdot \rvz)$. For a batch $\mathcal{B}_t = \{(z_n,y_n)\}_{n=1}^B$, the batch-averaged loss is $L_t = -\frac{1}{B} \sum_{n=1}^B \log p_{\vtd}(y_n|\rvz_n)$. With Algorithm \ref{alg:enc} and \ref{alg:dec} applied for encoding and decoding, the description length of the compressed $\mathcal{B}_t$ is $\sum_{n=1}^B\left[-\log p_{\vtd} (y_n|\rvz_n)\right] = B\cdot L_t$, which means Theorem \ref{trm:compression} still holds and the discussion in Appendix \ref{app:our_compression} also applies. For a dataset with $N$ examples in total, the compression ratio $\CR \approx \frac{N\log V}{\sum_{t=1}^{N/B} B\cdot L_t} = \frac{N}{\sum_{t=1}^{N/B} B\cdot L_t}$. For a randomly initialized $\vtheta_1$, $L_1\approx 1$, and as the model trains, $L_t \rightarrow 0$, indicating a valid data compressing process of compression ratio $\CR > 1$.

\label{app:perceptron_as_comp}

\section{Proof of Theorem \ref{trm:main}}
\label{app:derive}
Theorem \ref{trm:main} essentially reflects the property of the dynamics in the learning process induced by the optimal learning policy $\vga(t)$ for the problem defined in Equation \ref{eq:obj}.
We choose the accumulation of each $\gnc$ over time: $\Gnc = \int_0^t \gamma_n(t') \rmd t'$, as a set of free variables that $\vtc$ depends on to solve the optimization problem. In this way, the problem is simplified by considering a scalar that summarizes ``how much'' an example is used for training until $t$, rather than the whole trajectory of $\gnc$. As such, $\dot{\Gamma}_n(t) = \frac{\rmd }{\rmd t}\Gnc=\gnc$ and Equation \ref{eq:obj} becomes:
\begin{equation}
    \begin{aligned}
    \min_{\vGamma(t),\vga(t)} \ \ & \int_0^T L^{\text{dsr}}(\vtheta_{\vGamma, \vga}(t)) \rmd t, \\
    \text{s.t.} \ \ 
    & \sum_{n=1}^N \dot{\Gamma}_n(t) = 1, \\
    & \dot{\Gamma}_n (t) \ge 0, n=1,2,\cdots,N,
    \label{eq:obj2}
    \end{aligned}
\end{equation}
where $\vGamma(t) = \left[\Gamma_1(t),\Gamma_2(t),\cdots,\Gamma_N(t)\right]^\top$, and $\vtheta_{\vGamma,\vga}(t)$ is an alias of $\vtc$ to show its dependency on $\vGamma$ and $\vga$. Let $\mathcal{L}$ be the Lagrangian depending on $\{\Gnc\}_{n=1}^N$ and $\{\dot{\Gamma}_n(t)\}_{n=1}^N$:
\begin{equation}
 \begin{aligned}
    \mathcal{L} = \simpldsrc + \lambda (t) (\sum_{n=1}^N \dot{\Gamma}_n(t) - 1) + \sum_{n=1}^N \mu_n(t) \dot{\Gamma}_n (t),
\end{aligned}
\label{eq:lag2}
\end{equation}
where $\lambda(t)$ and $\mu_n(t)$ are Lagrange multipliers and $\simpldsrc = L^{\text{dsr}}(\vtheta_{\vGamma, \vga}(t))=\ldsrc$. To achieve the optimum of Equation \ref{eq:obj2}, $\mathcal{L}$ should satisfy the Euler-Lagrange (EL) Equation~\cite{el_equation}:
\begin{equation}
    \begin{aligned}
    \frac{\rmd }{\rmd t}\frac{\partial \mathcal{L}}{\partial \dot{\Gamma}_n} - \frac{\partial \mathcal{L}}{\partial \Gn} = 0.
    \end{aligned}
\end{equation}
Together with other constraints in the Karush–Kuhn–Tucker (KKT) conditions~\cite{kkt}, we get the following formulas that characterize the optimum of the Equation \ref{eq:obj2}:
\begin{equation}
    \label{eq:kkt}
    \left\{
    \begin{aligned}
    \frac{\partial \simpldsrc}{\partial \Gn} &= \dot{\lambda}(t)+\dot{\mu}_n(t), \\
    \sum_{n=1}^N \dot{\Gamma}_n(t) &= 1, \\
    \dot{\Gamma}_n(t) &\ge 0, \\
    \mu_n(t) &\ge 0, \\
    \mu_n(t)\dot{\Gamma}_n(t) &= 0.
    \end{aligned}
    \right.
\end{equation}
Note that we only consider the E-L Equations on $\{\Gnc\}_{n=1}^N$ and $\{\dot{\Gamma}_n(t)\}_{n=1}^N$, part of the free variables in the original problem, which also depends on the $\vga(t)$ trajectory. Since KKT conditions are necessary conditions to the optimization problem, Equation \ref{eq:kkt} is also necessary. In the following, we simplify Equation \ref{eq:kkt} to reach Theorem \ref{trm:main}.

\paragraph{Simplifying Equation \ref{eq:kkt}.} We study the examples with non-zero weights during training. For $\gnc=\dot{\Gamma}_n(t)>0$ we have $\mu_n(t)=0$ according to $\mu_n(t)\dot{\Gamma}_n(t) = 0$ in Equation \ref{eq:kkt} and the following Lemma bridges a connection between $\mu_n(t)$ and $\dot{\mu}_n(t)$:
\begin{lemma}
    \label{lem:zero}
    For $\mu_n(t)\in C^1[0,T]$, $\mu_n(t)=0 \Rightarrow \dot{\mu}_n(t)=0$ at the specific time step $t$.
\end{lemma}
\begin{proof}
    Assuming $\exists t_0\in [0,T]$, s.t. $\mu_n(t_0)=0$ but $\dot{\mu}_n(t_0)\ne 0$, we let $\dot{\mu}_n(t_0)> 0$ without loss of generality. Then $\exists \delta t >0$, s.t. $\mu_n(t_0-\delta t) < 0$ according to $\mu_n(t) \in C^1[0,T]$, which contradicts $\mu_n(t) \ge 0$ in Equation \ref{eq:kkt}. Therefore, we have $\dot{\mu}_n(t_0) = 0$.
\end{proof}
As such, $\gnc > 0 \Rightarrow \mu_n(t)=0 \Rightarrow \dot{\mu}_n(t)=0$ (Lemma \ref{lem:zero}), which means:
\begin{equation}
\label{eq:temp_main}
\frac{\partial \simpldsrc}{\partial \Gamma_m} =
\frac{\partial \simpldsrc}{\partial \Gn} = \dot{\lambda}(t), \text{  for } \gamma_m(t) > 0 \text{ and } \gnc > 0.
\end{equation}
Note that $\dot{\lambda}(t)$ is independent of $m$ and $n$. Equation \ref{eq:temp_main} already resembles Equation \ref{eq:main} in their formats, if we have $\frac{\partial \simpldsrc}{\partial \Gn} \propto \nabla L \cdot \nabla l_n$, where $\nabla L = \nabla \ldsrc$ and $\nabla l_n = \nabla l(x_n^\trn, \vtc)$.

\paragraph{Interpreting $\bm{\frac{\partial \simpldsrc}{\partial \Gn}}$.} $\frac{\partial \simpldsrc}{\partial \Gn}$ measures how the change of $\Gnc$ influence the change of $\simpldsrc$ at the time step $t$ when other free variables are fixed. Specifically, if $\Gnc$ changes by a small value $\Gnc \rightarrow \Gnc + \Delta \Gnc$, then $\simpldsrc$ correspondingly changes by a small value $\simpldsrc \rightarrow \simpldsrc + \Delta \simpldsrc$, and $\frac{\partial \simpldsrc}{\partial \Gn} = \frac{\Delta \simpldsrc}{\Delta \Gnc}$. Then, we consider $\frac{\rmd \simpldsrc}{\rmd t}$ with Equation \ref{eq:continuous}:
\begin{equation}
    \label{eq:loss_dyna}
    \begin{aligned}
    \frac{\rmd \simpldsrc}{\rmd t} &= \nabla \ldsrc \cdot \frac{\rmd \vtc}{\rmd t} \\
    &= - \sum^{N}_{n=1}\gamma_{n}(t) \nabla \ldsrc \cdot \nabla l(x_n^\trn, \vtc) \\
    &= - \sum^{N}_{n=1}\frac{\rmd \Gnc }{\rmd t} \nabla L \cdot \nabla l_n.
    \end{aligned}
\end{equation}
As a result, for a small $\Delta t$, we have:
\begin{equation}
    L^{\text{dsr}}(t+\Delta t) - \simpldsrc = - \sum^{N}_{n=1}\left[\Gn(t+\Delta t) - \Gnc\right] \nabla L \cdot \nabla l_n.
\end{equation}
Now we consider the change of $L^{\text{dsr}}(t)$ and $\Gn$ at $t+\Delta t$. Since $\nabla L \cdot \nabla l_n$ is computed at the time step $t$, it is not affected by the variants. Therefore, we get:
\begin{equation}
    \begin{aligned}
        \Delta L^{\text{dsr}}(t+\Delta t) = -\Delta \Gn(t+\Delta t) \nabla L \cdot \nabla l_n,
    \end{aligned}
\end{equation}
When $\Delta t \rightarrow 0$, $\frac{\Delta L^{\text{dsr}}(t+\Delta t)}{\Delta \Gn(t+\Delta t)} \rightarrow \frac{\Delta \simpldsrc}{\Delta \Gnc} = \frac{\partial L^{\text{dsr}}}{\partial \Gn}$, which means:
\begin{equation}
    \label{eq:if}
    \frac{\partial \simpldsrc}{\partial \Gn} = - \nabla L \cdot \nabla l_n.
\end{equation}
By substituting Equation \ref{eq:if} into Equation \ref{eq:temp_main}, we obtain that for the $m^{\text{th}}$ and $n^{\text{th}}$ training examples satisfying $\gamma_m(t) > 0$ and $\gnc > 0$ the following equation holds:
\begin{equation}
    \label{eq:temp_main_2}
     \nabla L \cdot \nabla l_m = \nabla L \cdot \nabla l_n = -\dot{\lambda}(t) = \Const,
\end{equation}
where $\Const$ stands for ``a constant independent of $m$ and $n$''. Equation \ref{eq:temp_main_2} is essentially equivalent to Equation \ref{eq:main}.

\paragraph{Proving $\bm{\Const = -\frac{\rmd L^\dsr(t)}{\rmd t}}$.} By substituting $\nabla L \cdot \nabla l_n$ with $\Const$ in Equation \ref{eq:loss_dyna}, we get:
\begin{equation}
    \begin{aligned}
    \frac{\rmd L^\dsr(t)}{\rmd t} &= - \sum^{N}_{n=1}\frac{\rmd \Gnc }{\rmd t} \cdot \Const, \\
    &= -\Const \sum_{n=1}^N \gnc, \\
    &= -\Const.
    \end{aligned}
    \label{eq:const}
\end{equation}
As such, by combining Equation \ref{eq:temp_main_2} with Equation \ref{eq:const}, we complete the proof of Theorem \ref{trm:main}.

\begin{figure}[t]
    \centering
    \includegraphics[width=0.4\textwidth]{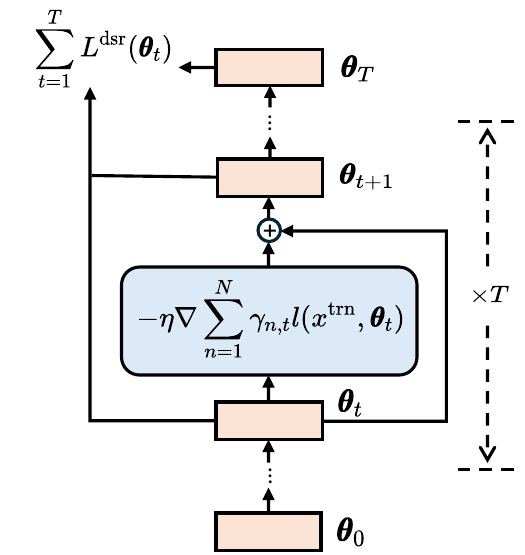}
    \caption{The architecture of the equivalent neural network to find the optimal learning policy, Each layer consists of the gradient update and a residual connection.}
    \label{fig:opt_gamma}
\end{figure}

\section{Details of Learning Policy Optimization}
\label{app:policy_opt}
In Section \ref{sec:method}, we search for the optimal learning policy by Proximal Gradient Decent~\cite{prox_gd}. Specifically, we view the whole learning process in $0 \le t \le T$ as a neural network with $T$ layers parameterized by $\vga = [\vga_{0}, \cdots, \vga_{t-1}] \in \mathbb{R}^{N \times T}$. As illustrated in Figure \ref{fig:opt_gamma}, each layer of the network consists of the gradient update function and a residual connection~\cite{res_net}, where the ``hidden states'' are $\vtd$. Then, we adopt Backward Propagation (BP;~\citealp{bp}) to compute $\nabla_{\vga_t} J(\vga)$ in Equation \ref{eq:method}. The backward operation at each layer is:
\begin{equation}
    \begin{aligned}
    \frac{\partial J}{\partial \vga_t} &= -\eta\sum_{t'=t+1}^T \nabla L^{\dsr}(\vtheta_{t'})
    \frac{\partial \vtheta_{t'}}{\partial \vtheta_{t+1}} \rmG^{\trn}(\vtd)
    \\
    \frac{\partial \vtheta_{t'}}{\partial \vtheta_{t+1}} &= \frac{\partial \vtheta_{t'}}{\partial \vtheta_{t+2}} \left[\mI - \eta \rmH^\trn(\vtheta_{t+1})\right],
    \end{aligned}
\end{equation}
where $\rmG^\trn(\vtd)=\left[\nabla l(x_1^\trn, \vtd), \cdots, \nabla l(x_N^\trn, \vtd)\right] 
$
, $\mI$ is the identity matrix, and $\rmH^\trn(\vtheta_{t+1})$ is the Hessain matrix of $L^\trn(\vtheta)$ at $\vtheta = \vtheta_{t+1}$. We implement the BP operations with dynamic programming and Jacobian-Vector-Product\footnote{\url{https://pytorch.org/docs/stable/func.api.html}} in PyTorch~\cite{pytorch} for efficiency. To reduce the single-device GPU memory use, we also implement an activation partition algorithm inspired by ZeRO-2~\cite{zero}, where the ``hidden states'' $\vtd$ in one model are stored in different GPU devices.

\section{Hyper-Parameter Configurations}
\label{app:hp}

\paragraph{Perceptron Linear Classification.} Following the teacher-setting described in Section \ref{sec:exp_setup}, we use $D=128$ and $\rmT$ is randomly drawn from a Gaussian distribution $\rmT \sim \mathcal{N}(\bm{0}, \sqrt{D}\mI)$. We generate $N=4096$ training inputs $\rvz^\trn$ from $\mathcal{N}(\bm{0}, 3\mI)$, and $M=512$ target inputs $\rvz^\dsr$ from $\mathcal{N}(0.5\mathbf{1}, \mI)$, where $\mathbf{1}=[1,1,\cdots,1]^\top \in \mathbb{R}^D$. For each learning policy, we initialize $\gamma_{n,0}=\frac{1}{N}$ and train the Perceptron with $\eta=0.1$ for $T=2000$ time steps, which is sufficient for the model to converge. For learning policy optimization, we initialize the learning policy to the constant policy $\gamma^c_{n,t}=\frac{1}{N}$, setting $\epsilon=5\times 10^{-6}$ and train the network for 500 epochs.

\paragraph{Transformer Language Modeling.} We conduct experiments based on a two-layer Transformer with 128 hidden dimensions and 8 attention heads. For all experiments except that in Table \ref{tab:scaling_law}, we randomly sample $N$ = 16,384 examples as $x_n^\trn$ and $K$ = 512 examples as $x_k^\dsr$ with the max sequence length 64 from the TinyStories~\cite{tinystories} corpus\footnote{\url{https://huggingface.co/datasets/roneneldan/TinyStories/tree/main}}. 
We use the BPE tokenizer of Mistral~\cite{mistral} and construct a vocabulary with 5K tokens by mapping the infrequent tokens to [UNK]. The model contains about 1.7M parameters. To reflect the difference between the training and desired distribution, we add perturbations to 50\% training sentences by iteratively applying one of the following operations 20 times: 
(1) replacing one token with a random token in the vocabulary~\cite{top_p}; 
(2) deleting the last token~\cite{over_smoothing}; 
(3) repeating one token in a sentence~\cite{unlikelihood_training}. 
This corresponds to the fact that the large-scale pre-training corpus tends to be more noisy than the desired set (the carefully curated held-out corpus or high-quality downstream data) to evaluate the model generalization performance in practice. We set $\eta=0.1$, $T=4,000$, $\gamma_{n,0} = \frac{1}{N}$ for each learning policy. We start from the constant policy and optimize the learning policy for 15 epochs using $\eta=0.1, 0.2, 0.4$. $\eta=0.4$ yields the lowest loss at the end of the training. Therefore, we only plot the optimization process for $\eta=0.4$ in Figure \ref{fig:exp_opt_trm} and \ref{fig:exp_train_trm}. For experiments in Table \ref{tab:scaling_law}, we vary the total training steps and the corresponding training data sizes and simultaneously, change the vocabulary sizes to adapt to different data sizes. We use vocabulary sizes 4K, 4.5K, 5K, and 6K for $N=2^{12}, 2^{13}, 2^{14}, $ and $2^{15}$, respectively.

\end{document}